\RequirePackage{fix-cm}  
\documentclass[smallcondensed]{svjour3}     
\usepackage{etex} 
\smartqed  
\usepackage{graphicx}

 \makeatletter \let\cl@chapter\relax \makeatother

\usepackage{textcomp}
\usepackage{alltt}
\usepackage{textcomp}
\usepackage{url}

\usepackage{verbatim}
\usepackage{setspace}
\usepackage{cancel}
\usepackage{verbatim}
\usepackage{bbm}
\usepackage{bm}
\usepackage{amsmath}
\usepackage{amsfonts}
\usepackage{pifont}
\usepackage{amssymb}
\usepackage{algpseudocode}
\usepackage{pseudocode}
\usepackage[caption=false]{subfig}
\usepackage{bbm}
\usepackage{multicol}
\usepackage{tikz}
\usepackage{pgf}
\usetikzlibrary{arrows,automata}
\usetikzlibrary{positioning}
\usetikzlibrary{shadows,shapes,decorations}
\usepackage[all]{xy}
\usetikzlibrary{calc,fit}
\usepackage[colorinlistoftodos]{todonotes}

\DeclareMathOperator*{\argmax}{arg\,max}
\newcommand\independent{\protect\mathpalette{\protect\independenT}{\perp}} \def\independenT#1#2{\mathrel{\rlap{$#1#2$}\mkern2mu{#1#2}}}

\newcommand{\T}{\textnormal{T}}
\newcommand{\dr}{\textnormal{d}}

    \usepackage{soul} 

 \journalname{Annals of Mathematics and Artificial Intelligence}

\usepackage{hyperref}
\begin{document}

\title{A symbolic algebra for the computation of expected utilities in multiplicative influence diagrams\thanks{ Manuele Leonelli  was funded by Capes, whilst Jim Q. Smith was partly funded by the EPRSC grant EP/K039628/1.}
}

\titlerunning{Symbolic computation of expected utilities in influence diagrams}        

\author{Manuele Leonelli         \and
        Eva Riccomagno \and Jim Q. Smith 
}


\institute{M. Leonelli \at
              Instituto de Matem\'{a}tica, Universidade Federal do Rio de Janeiro, RJ, Brazil \\
              Tel.: +5521981648148\\
              \email{manuele@dme.ufrj.br}           
           \and
           E. Riccomagno \at
              Dipartimento di Matematica, Universita' degli Studi di Genova, Via Dodecaneso 35, 16146 Genova, Italia
           \and 
           J. Q. Smith \at
           Department of Statistics, The University of Warwick,
CV47AL Coventry, UK    
}

\date{Received: date / Accepted: date}

\maketitle

\begin{abstract}
Influence diagrams provide a compact graphical representation of decision problems. Several algorithms for the quick computation of their associated expected utilities  are available in the literature. However, often they rely on a full quantification of both probabilistic uncertainties and utility values. For problems where all random variables and decision spaces are finite and discrete, here we develop a symbolic way to calculate the expected utilities of influence diagrams that does not require a full numerical representation. Within this approach expected utilities correspond to families of polynomials. After characterizing their polynomial structure, we develop an efficient symbolic algorithm for the propagation of expected utilities through the diagram and provide an implementation of this algorithm using a computer algebra system. We then characterize many of the standard manipulations of influence diagrams as transformations of polynomials. We also generalize the decision analytic framework of these diagrams by defining asymmetries as operations over the expected utility polynomials.
\end{abstract}
\section{Introduction}
\label{introduction}
Decision makers (DMs) are often required  to choose in critical situations between a wide range of different alternatives. They need to consider the mutual influence of  quantifications of different types of uncertainties, the relative values of competing objectives together with the consequences of the decisions they will make. They can thus benefit from an intuitive framework which draws together these uncertainties and values so as to better understand and evaluate the full consequences of the assumptions they are making. To this end, a variety of graphical models have been developed. The most important of these are \textit{Bayesian networks} (BNs)~\cite{Pearl1988,Smith2010} and \textit{influence diagrams} (IDs)~\cite{Bielza2011,Howard1983,Jensenbook}, both of which provide an intuitive qualitative representation of the elements of the DM's problem together with  relatively fast computational tools for the calculation of, respectively, probabilities and expected utilities (EUs)~\cite{Jensen1994,Shachter1986,Smith2010}. Although only the second class of models can be used to automatically select an optimal course of action, i.e. an EU maximizing decision, both BNs and IDs are invaluable decision support tools, enabling DMs to easily investigate the effect of their inputs to an output of interest. 

Most of the algorithms for the computation of probabilities and EUs rely on a full specification of the model's parameters. Furthermore, commonly available software almost exclusively work numerically with complete elicitations. However, often in practice DMs  might not be confident about the precision of their specifications, nor have available all such values. This may lead  to non-robust decision making where the efficacy of decisions can change under small perturbations of the model's inputs. Symbolic approaches, not requiring full elicitations of the parameters, have proven useful in performing these types of input-output investigations, usually called \textit{sensitivity analyses}, both in fully inferential and decision making contexts ~\cite{Bhattacharjya2008,Bhattacharjya2010,Chan2004,Nielsen2003}. A variety of symbolic methods for both inference and sensitivity analysis are now in place for BNs~\cite{Castillo1997a,Castillo2003}. However, the development of symbolic techniques for EU computations in IDs has been largely neglected. An exception is a recent paper~\cite{Borgonovo2014} where \textit{decision network polynomials} are defined in the context of Bayesian decision problems. These are piece-wise functions made of so-called \textit{pieces}: multilinear polynomials having as indeterminates both probability and utility parameters. A new symbolic sensitivity technique is then developed in~\cite{Borgonovo2014} based on differentiation and difference operators. 

In this paper, we focus on a large class of IDs called \textit{multiplicative influence diagrams} (MIDs), which include as a special case standard IDs equipped with additive utility factorizations, and fully characterize the polynomial structure of the EU pieces (Section~\ref{influence}). We then introduce a symbolic algorithm for their computation, based on simple matrix operations (Section~\ref{algebra}), and its implementation in the computer algebra system Maple\textsuperscript{\textit{\tiny{TH}}}\footnote{Maple is a trademark of Waterloo Maple Inc.} (Appendix~\ref{maple}). Because of the simplicity of the required operations, our algorithm is shown to have computational times comparable to those of standard numerical evaluation software for graphical models (Section~\ref{sect:simulation}).  In contrast to standard software, which assumes an additive factorization between utility nodes, we also explicitly analyze cases when the more general class of multiplicative utility functions might be necessary~\cite{Keeney1974,keeney93,Smith2010}. We concentrate our study on the class of multiplicative factorizations because this provides some computational advantages over, for example, the more general class of multilinear utilities~\cite{keeney93}, whilst allowing for enough flexibility to model the DM's preferences in many real cases~\cite{French1989,Keeney1974}. This factorization turns out to be particularly efficient since it leads to a distributed  propagation of EUs as shown in Proposition~\ref{i-th}.

 The symbolic definition of the ID's probabilities and utilities in Section~\ref{influence} provides an elegant and efficient embellishment of the associated graphical representation of the decision problem, around which symbolic computations can then be carried out. In Sections~\ref{topology} and~\ref{asymmetry} standard manipulations of IDs and asymmetries are characterised on this new polynomial representation.  Importantly we demonstrate that, whilst graphical representations of asymmetries are rather more obscure than standard ID models, in our symbolic approach the imposition of asymmetries greatly simplifies the polynomial representation of the problem. The example in Section~\ref{sec:example} then outlines the insights our approach can give to DMs through the comparison of different parameters' specifications. Our symbolic approach has the great advantage in such sensitivity studies that, by exploiting the known polynomial expression of the problem, one can simply plug-in different numerical specifications and instantaneously get the EU values. In standard numerical approaches on the other hand, the evaluation algorithm needs to be run for each combination of parameters considered. This can become quickly unfeasible even for rather small problems.

\section{A review of symbolic approaches to decision making and support}
Symbolic inference and decision support techniques have already been used for the analysis of BN models. A symbolic definition of probabilities in BNs in terms of multilinear polynomials first appeared in~\cite{Castillo1995}. Since then various inferential techniques have been developed~\cite{Castillobook,Darwiche2003,Gorgen2015}. Their most demonstrably useful application is in the process of validating models through sensitivity analyses. Two main approaches are adopted in practice. The first one is based on differentiation of the probability polynomials and is useful for  the analysis of global changes of  probability distributions~\cite{Chan2001,Chan2004}. The second one concerns local changes studied via sensitivity functions~\cite{Coupe2002,Van2007}, which, because of the  assumed multilinearity, are simple linear functions of the parameters of interest. Recently, symbolic methods have been extended to asymmetric models~\cite{Gorgen2015,Leonelli2015r} where the associated polynomials might not exhibit regular multilinear structures as for BNs.

Although it is known that EUs in IDs also have a multilinear structure~\cite{Felli2004}, symbolic methodologies for such models have not been studied consistently. Only recently the robustness of decision models has been analysed from a symbolic viewpoint in~\cite{Borgonovo2014}. For the i-th available strategy, \cite{Borgonovo2014} defines the functions $u_i:\mathcal{X}\rightarrow \mathbb{R}$, where $\mathcal{X}$ is the parameter space, representing the EU of the associated strategy and called EU piece. The decision network polynomial is then defined as $\max_{i=1,\dots,M}u_i(\bm{x})$, for $M$ available strategies, and represents the expected utility of the optimal strategy for the combination of parameters $\bm{x}$.

However, the typical assumptions of an applied decision analysis about the form of the utility function, often encoded via an ID representation,  are not utilized in~\cite{Borgonovo2014} and no details on how to compute the functions $u_i$ are given there. In this paper we extend the symbolic framework of~\cite{Borgonovo2014} by  developing a distributed symbolic procedure for the computation of the EU pieces for utilities chosen in the large class of multiplicative  IDs~\cite{keeney93,Smith2010}.  We further fully characterize symbolically the functions $u_i$ of the decision problem. This enables the application of the proposed methodology to robustness studies where certain parameters are treated as unknown. In Section~\ref{sec:example} via an example we show how to exploit our definition for informing a DM about the optimization process. A full  development of such symbolic optimization techniques is  beyond the scope of this paper.

Of course the solution and investigation of both generic decision problems and influence diagrams can be performed outside of the full Bayesian symbolic paradigm and using uncertainty calculi that relax the assumption of an exact and complete probability specification. One of such proposals~\cite{De2008}, is based on imprecise probabilities and consists of mapping the evaluation of an ID into an inferential problem in credal networks~\cite{Cozman2000}, solved using multilinear programming~\cite{De2004}. The objective function of such an optimization problem can be shown to be multilinear and to share many features with our polynomial representation of EUs, although within a different domain. Because of the use of imprecise probabilities the parameters of the decision problem can be specified only partially.  

Symbolic evaluation methods have also been  introduced for discrete and finite time decision Markov processes that do not require full parameters' elicitations (e.g.~\cite{Kikuti2011}).  As an ID can always be cast as a Markov decision process, the evaluation methods originally designed for general Markov processes can be straightforwardly applied to IDs.  A different approach is taken by the so called symbolic dynamic programming: for such a technique the sample space does not need to be fully specified~\cite{Sanner2010,Zamani2012}. Again these methods have the capability of  helping the DM to discover the most critical features of the decision problem where accurate specification of inputs is most necessary. 

The methods reviewed above propose to automate decision making in a variety of frameworks and reasoning paradigms where DMs do not need to provide complete and/or exact parameters' specifications. These have proven to be successful and computationally efficient, but EU maximization is still most commonly applied within a standard probabilistic domain. Therefore, here we assume that the DM plans to behave as an EU maximizer and we will henceforth work entirely within this most standard framework.

\section{Symbolic representation of influence diagrams}
\label{influence}
In this paper, with the exception of Section~\ref{asymmetry}, we consider those Bayesian decision problems that can be represented by an ID and are usually called \textit{uniform} (or \textit{symmetric})~\cite{Koller2009,Smith2010}.   Let $n$ be a positive integer ($n\in \mathbb Z_{\geq 1}$) and $\mathbb{D}$ and $\mathbb{V}$ be a partition of  $[n]=\{1,\dots,n\}$. Let $\{Y_i:i\in\mathbb{D}\}$ be a set of controlled (or decision)\footnote{With controlled variable we mean a variable set by the DM to take a particular value.} variables and $\{Y_i:i\in\mathbb{V}\}$ a set of non-controlled (or random) variables. As in standard ID representations, the set of decision variables  is assumed to be totally ordered and the union of $\{Y_i:i\in\mathbb{V}\}$ and $\{Y_i:i\in\mathbb{D}\}$ to be totally ordered compatibly with a partial order on the random variables. Let $\preceq$ be the chosen ordering relationship.  The ordering on the $Y_i$'s is reflected by their indices, that is if $Y_i\preceq Y_j$ then $i<j$. 

For $i\in [n]$ and $r_i\in\mathbb{Z}_{\geq 1}$, let $[r_i]_0=\{0,\ldots,r_i-1\}$ and $Y_i$ take values in $\mathcal{Y}_i=[r_i]_0$. For $A\subseteq [n]$,  let the vector $\bm{Y}_{A}=(Y_i)_{i\in A}$ take values in $\bm{\mathcal{Y}}_{A}=\times_{i\in A}\mathcal{Y}_i$  and denote with $\bm{y}_A$ a generic instantiations of $\bm{Y}_A$. Examples of this notation are: the vector $\bm{Y}_{[n]}$ includes all the variables, whilst $\bm{Y}_{\mathbb{D}}$ and $\bm{Y}_{\mathbb{V}}$ are the vectors of controlled and random variables respectively.

\begin{figure}
\vspace{0.4cm}
\centerline{
\xymatrix{
*+[F]{Y_1}\ar@/^2pc/[rr]\ar[r]\ar[rd]&*+[Fo]{Y_3}\ar[dl]\ar@/^2pc/[rr]\ar[r]&*+[F]{Y_4}\ar[d]\ar[r]\ar[rd]&*+[Fo]{Y_5}\ar[r]\ar[d]&U_2\\
U_1&*+[Fo]{Y_2}\ar[u]\ar[ru]&U_3&*+[Fo]{Y_6}\ar[l]&
}
}
\caption{An MID consisting of two decision nodes, $Y_1$ and $Y_4$, four random nodes, $Y_2$, $Y_3$, $Y_5$ and $Y_6$, and three utility nodes, $U_1$, $U_2$ and $U_3$.}
\label{fig-ex}
\end{figure}
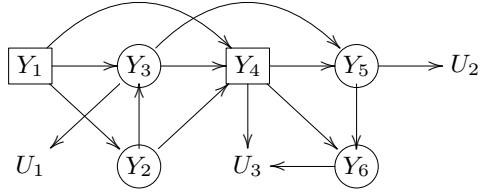

\subsection{Multiplicative influence diagrams}

We consider the class of \emph{multiplicative} IDs entertaining a multiplicative factorization over the utility nodes $\bm{U}=(U_1,\dots,U_m)^\T$ (see e.g.~\cite{keeney93,Smith2010}). For $i\in [m]$, $U_i$ is a function onto $[0,1]$ defined on a subspace $\bm{\mathcal{Y}}_{P_i}$ of $\bm{\mathcal{Y}}_{[n]}$ where $P_i\subseteq [n]$ is assumed non empty.

\begin{definition}
\label{mid}
A  \textbf{multiplicative influence diagram} (MID) $G$ consists of three components:  a directed acyclic graph (DAG) with vertex (or node) set
$V(G)= \bm{Y}_{[n]} \cup \bm{U}$,  
a transition probability function related to the random variables $\bm{Y}_{\mathbb{V}}$ and  a multiplicative factorization function related to the $\bm{U}$ nodes. 
\end{definition}

\begin{example}
 Fig.~\ref{fig-ex} presents an MID with $n=6$, $m=3$, $\mathbb{D}=\{1,4\}$, $\mathbb{V}=\{2,3,5,6\}$ and vertex set $V(G)=\{Y_1,\dots, Y_6, U_1,\dots, U_3\}$. 
There are two controlled variables, $Y_1$ and $Y_4$,  four random variables, $Y_2$, $Y_3$, $Y_5$ and $Y_6$, and three utility nodes, $U_1$, $U_2$ and $U_3$. 
We adopt the convention by which decision variables and random variables are respectively framed with squares and circles. 
All variables are binary and take values in the spaces $\mathcal{Y}_i=\{0,1\}$, $i\in [6]$. 
\end{example}

Next we describe the three components of an  MID starting from its  edge  (or arc) set $E(G)$.  For $i\in [n]$,  the parent set of $Y_i$ is the sub-vector of $\bm{Y}_{[n]}$ indexed by  $\Pi_i \subset[i-1]$.  For  $i\in [m]$,  the parent set of $U_i$ is the sub-vector  $\bm{Y}_{P_i}$ of $\bm{Y}_{[n]}$ where $P_i\subseteq [n]$ is the non empty set mentioned above and  thus each utility node has at least one parent. Furthermore any two $P_i$'s are assumed disjoint so that each component of $\bm{Y}_{[n]}$ is parent of at most one utility node. There are three types of edges in an MID:
\begin{enumerate}
	\item \label{aa} those into $\bm{U}$ vertices: 
	for $i\in [m]$, $U_i$ has no children and its parent set  $\bm{Y}_ {P_i}$ is described above; 
	\item those into $\mathbb{D}$ vertices: for $i\in\mathbb{D}$,  the parent set of $Y_i$ consists of the variables, controlled and non-controlled that are known when $Y_i$ is controlled; 
	\item those into $\mathbb{V}$ vertices: for $i\in \mathbb{V}$, 
	 the parent set of $Y_i$ is such that $Y_i$ is conditionally independent (with respect to the probability law in Definition~\ref{mid}) of the random variables preceding it given its parents  and for all instantiations of decisions preceding $Y_i$. 
	 \end{enumerate}
Recalling that $\Pi_i\subset [i-1]$,  Item (3) above can be formulated as  
 $Y_i \independent \bm{Y}_{[i-1] } \;|\; \bm{Y}_{\Pi_i}$,
where $\independent$ denotes the \textit{extended} conditional independence operator~\cite{Dawid2014}.
 This means that standard conditional independence,
namely $Y_i \independent \bm{Y}_{ [i-1] \cap \mathbb{V}} \;|\; \bm{Y}_{\Pi_i \cap\mathbb{V}}$,  
holds for all instantiations of the decision variables $\bm{Y}_{ [i-1] \cap \mathbb{D}} $
preceding $Y_i$. The transition probability function for the random vector $\bm{Y}_{\mathbb{V}}$ in Definition~\ref{mid} is given in terms of probability density as the product of  
$P_i(y_i\;|\;\bm{y}_{\Pi_i})=P(Y_i=y_i\;|\;\bm{Y}_{\Pi_i}=\bm{y}_{\Pi_i})$ 
for $i\in\mathbb{V}$. Note that $\bm{y}_{\Pi_i}$ includes instantiations of controlled variables as well as random variables. 
  \begin{example}
The edge set of the MID in Fig. \ref{fig-ex} is such that no variable is observed before controlling $Y_1$, whilst $Y_1$, $Y_2$ and $Y_3$ are observed before controlling $Y_4$ since $\Pi_4=\{1,2,3\}$.  
Furthermore its DAG implies that $Y_5\independent Y_1,Y_2\;|\;Y_3,Y_4$ and $Y_6\independent Y_1,Y_2,Y_3\;|\; Y_4,Y_5$.
The parent sets of the utility nodes are $P_1=\{3\}$, $P_2=\{5\}$ and $P_3=\{4,6\}$.
 \end{example}
 The third component of an MID is a utility function $U$ defined over $\bm{\mathcal{Y}}_{[n]}$ as 
\begin{equation}
\label{muia}
U(\bm{y}_{[n]})=\left\{
\begin{array}{ll}
\sum_{i\in[m]}k_iU_i(\bm{y}_{P_i}), & \mbox{if } h=0, \\
\sum_{I\in \mathcal{P}_0([m])}h^{n_I-1}\prod_{i\in I} k_iU_i(\bm{y}_{P_i}),& \mbox{otherwise.} 
\end{array}
\right.
\end{equation}
where 
$k_i\in (0,1)$ is a \textit{criterion weight}~\cite{keeney93}; 
as mentioned above 
$U_i$ is a function of the random and decision variables in $\bm{Y}_{P_i}$.
It gives the contribution to the utility function of  the controlled and random variables in $\bm{Y}_{P_i}$ and it does so linearly if $h=0$, i.e. the first case of equation (\ref{muia}). 
It is worthwhile recalling that the $\bm{Y}_{P_i}$'s are disjoint.
In the second case of equation (\ref{muia}) 
 $h$ is the unique non-zero solution not smaller than minus one to
\begin{equation}
\label{k}
1+h=\prod_{i\in[m]}(1+hk_i).
\end{equation} 
and 
$\mathcal{P}_{0}(\cdot)$ denotes the power set without the empty set, $n_{I}$ is the number of elements in the set $I$. 
For $h=0$, the multiplicative factorisation of an MID, $U(\bm{y}_{[n]})$, is a weighted sum of the terms $U(\bm{y}_{P_i})$: thus coinciding with the class of commonly used additive factorizations~\cite{Keeney1974}. Therefore the methodology we develop here applies to utility factorizations of additive form, or \textit{additive IDs}, as well. 
For $h\neq 0$  the function $U(\bm{y}_{[n]})$ is a linear combination of all square free products of the $U_i$'s (excluding $1$). 
The $h$ balances the weight of the interaction terms: the larger $h$ is, the bigger is the impact of high order terms. 

\begin{example}
The multiplicative utility factorization associated to the MID in Fig. \ref{fig-ex}, for $h\neq 0$ and leaving the functions' arguments implicit, can be written as
\[
U=k_1U_1+k_2U_2+k_3U_3+hk_1k_2U_1U_2+hk_1k_3U_1U_3+hk_2k_3U_2U_3+h^2k_1k_2k_3U_1U_2U_3.
\]
This expression emphasizes the generality of multiplicative utilities, since an additive utility is obtained by setting $h=0$ and is the sum of the first three terms. 
\end{example}

Item~\ref{aa} above, describing the edges into the utility nodes, extends the total order over $\bm{Y}_{[n]}$ to $V(G)$. Indeed for $i,j\in[m]$, $U_i$ succeeds $U_j$ and $i>j$ if there exists a parent of $U_i$ which succeeds all parents of $U_j$ in the order $\preceq$ over $\bm{Y}_{[n]}$: formally, if there is a $k\in P_i$ such that for every $l\in P_j$, $k>l$. For $i\in[m]$, let $j_i$ be the highest index of $P_i$ and $\mathbb{J}=\{j_1,\dots, j_m\}$. The set $\mathbb{J}$ of  the greatest parents of the utility nodes in $\preceq$ is fundamental for the Algorithm~\ref{algo} in Section~\ref{alg} because it allows for the computation of the least number of expected utilities by
processing a $U_i$ in the algorithm only when strictly necessary.  
The Maple\textsuperscript{\textit{\tiny{TH}}} function \texttt{CompJ} in Appendix~\ref{b1} computes the set $\mathbb{J}$ for a given MID.  The totally ordered sequence of $V(G)$ is called \textbf{decision sequence (DS)} of the MID $G$ and is denoted by  $S:=(Y_1,\dots,Y_{j_1},U_1,Y_{j_1+1},\dots,Y_{j_m},U_m)$. 
As in~\cite{Bhattacharjya2012}, we do not introduce utility nodes only at the end of the DS. This enables us to base the choice of optimal decisions, through the  algorithm given below,  only on the values of the relevant attributes.

\begin{example}\label{nuclearexample}
The DS associated to the MID  in Fig. \ref{fig-ex} is $(Y_1,Y_2,Y_3,U_1,Y_4,Y_5,U_2,Y_6,U_3)$ with $j_1=3$, $j_2=5$, $j_3=6$ and thus $\mathbb{J}=\{3,5,6\}$.  
\end{example}

\subsection{Evaluation of  MIDs} \label{subsec:evalMIDS}
 In this section we set the background for an efficient symbolic algorithm for  \textit{evaluating} an MID, namely for computing the expected value of equation~(\ref{muia}) for all possibile decisions $\bm{y}_{\mathbb D}\in {\mathcal\bm{ Y}}_{\mathbb D}$ and identifying a sequence of optimal decisions that maximizes it. 
We do this by exploiting the sequential structure of equation~(\ref{muia}) which by linearity is transferred to its  EU function.  
However, this can be done only for MIDs in \textit{extensive form} ~\cite{smith89}, namely those MIDs whose  topology is such that, for any index $j\in\mathbb{D}$, only variables that are known at the time the DM makes the  decision $Y_j$  have an index lower than $j$.
This is because the evaluation will output optimal decisions as functions of observed quantities only~\cite{Smith2010}. 
Extensive form is thus a property referring to the edges into the decision variables of an MID. 

\begin{definition}
\label{ei}
An MID $G$  is said to be in \textbf{extensive form} if $Y_i$ is a parent of $Y_j$, $j\in\mathbb{D}$, for all $i<j$. 
\end{definition}

\begin{example} \label{aa:example}
The MID in Fig.~\ref{fig-ex} is in extensive form since $\Pi_4=\{1,2,3\}$. If either the edge $(Y_2,Y_4)$ or $(Y_3,Y_4)$ were deleted then the MID would not be in extensive form. 
\end{example}

We first study MIDs in extensive form 
and  only in Section~\ref{topology} we consider manipulations of non extensive MIDs which turn them into extensive form. 
Without loss of generality we assume that any vertex corresponding to a variable in $\bm{Y}_{[n]}$ has at least one child. Indeed, random and controlled vertices with no children could simply be deleted from the graph without changing the outcome of the evaluation~\cite{Koller2009}.  
In Example~\ref{aa:example} the only vertices with no children are utility nodes. 

A typical way to evaluate an MID in extensive form is through a backward inductive algorithm  on the vertices of the DAG. We present a computationally  efficient  version of this algorithm, which at each step only  utilises the strictly necessary utility nodes.
The identification of the optimal policy is based on the computation of the functions $\bar{U}_i(\bm{y}_{B_i})$, $i\in[n]$, which are  formally introduced in Proposition~\ref{i-th} and each of which depends only on the variables in $\bm{Y}_{[n]}$ that are strictly required for an MID evaluation.  For $i\in [n]$, the set
\begin{equation*}
\label{Ai}
B_i=\left\{\bigcup_{\substack{k\geq i\\ k\in\mathbb{V}}}\Pi_k\bigcup\bigcup_{\substack{j\geq i\\j\in\mathbb{J}}}P_j\right\}\setminus\{i,\dots,n\},
\end{equation*}
defines the index sets of the subset of $\bm{Y}_{[n]}$ which appear as arguments of $\bar{U}_i$. The function \texttt{CompBi} in Appendix~\ref{b1}  computes the $B_i$'s given the definition of an MID. 
Specifically a set $B_i$ includes only indices smaller than $i$ that are either in the parent set of a random variable $Y_k$, $k>i$, following $Y_i$ in the DAG or in a set $P_j$ such that $U_j$ succeeds $Y_i$ in the DS of the MID.

\begin{example}
For the MID in Fig.~\ref{fig-ex}  the set 
 $B_5=\{3,4\}$  since $B_5=\{\Pi_6\cup \Pi_5\cup P_3\cup P_2\}\setminus\{5,6\}$, $\Pi_6=\{4,5\}$, $\Pi_5=\{3,4\}$, $P_3=\{4,6\}$ $P_2=\{5\}$, whilst $B_4=\{3\}$ since $B_4=\{\Pi_5\cup\Pi_6\cup P_2\cup P_3\}\setminus\{4,5,6\}=B_5\setminus \{4\}$.
\end{example}

 \begin{proposition}
\label{i-th}
The optimal decision associated to an MID yields EU equal to $\bar{U}_1(\bm{y}_{B_1})$ obtained with a backward recursion as follows. 
For $i\in[n]$ the function 
$\bar{U}_i(\bm{y}_{B_i})$ is defined according to whether $Y_i$ is a decision or a random variable as 
\begin{equation*}
\label{eq:id1}
\bar{U}_i(\bm{y}_{B_i})=\left\{
\begin{array}{lcl}
\bar{U}_{i,\mathbb{D}}(\bm{y}_{B_i}),&&\mbox{if } i\in\mathbb{D},\\
\bar{U}_{i,\mathbb{V}}(\bm{y}_{B_i}),&&\mbox{if } i\in\mathbb{V}
\end{array}
\right.
\end{equation*}
and three cases are distinguished
\begin{enumerate}
\item for $i=n$ either 
\begin{equation}  \label{cond4}
        \begin{aligned} 
	\bar{U}_{n,\mathbb{D}}(\bm{y}_{B_n}) & =\max_{\mathcal{Y}_n}k_mU_m(\bm{y}_{P_m}) 
	 \quad \text{or}  \\
	\bar{U}_{n,\mathbb{V}}(\bm{y}_{B_n}) & =\sum_{y_n\in\mathcal{Y}_n}k_mU_m(\bm{y}_{P_m})P_n(y_n\;|\;\bm{y}_{\Pi_n})
	\end{aligned}
\end{equation}  
\item for $i \in [n-1]$, $i\in \mathbb{J}$ and $i\in P_l$, then either 
\begin{equation}  \label{cond23}
        \begin{aligned} 
	\bar{U}_{i,\mathbb{D}}(\bm{y}_{B_i}) 
		&= \max_{\mathcal{Y}_i}\Big(hk_lU_l(\bm{y}_{P_l})\bar{U}_{i+1}(\bm{y}_{B_{i+1}})	
			+k_lU_l(\bm{y}_{P_l})+\bar{U}_{i+1}(\bm{y}_{B_{i+1}})\Big) \quad \text{or} \\  
	\bar{U}_{i,\mathbb{V}}(\bm{y}_{B_i}) 
		&= \sum_{y_i\in\mathcal{Y}_i}  \Big(hk_lU_l(\bm{y}_{P_l})\bar{U}_{i+1}(\bm{y}_{B_{i+1}}) + k_lU_l(\bm{y}_{P_l})  \\
			& \qquad\qquad + \bar{U}_{i+1}(\bm{y}_{B_{i+1}})\Big)P_i(y_i\;|\;\bm{y}_{\Pi_i}), 
	\end{aligned}
\end{equation} 
\item for $i\in[n-1]$ and $i\not\in\mathbb{J}$ either 
\begin{equation}  \label{cond1}
        \begin{aligned} 
        \bar{U}_{i,\mathbb{D}}(\bm{y}_{B_i}) & =\max_{\mathcal{Y}_i}\bar{U}_{i+1}(\bm{y}_{B_{i+1}}) \quad \text{or}  \\
        \bar{U}_{i,\mathbb{V}}  (\bm{y}_{B_i}) & =\sum_{y_i\in\mathcal{Y}_i}\bar{U}_{i+1}(\bm{y}_{B_{i+1}})P_i(y_i\;|\;\bm{y}_{\Pi_i}).
	\end{aligned}
\end{equation} 
\end{enumerate} 
\end{proposition}

All maxima and summations in Proposition~\ref{i-th} are over one $\mathcal{Y}_i$ sample space only.  
For example equation~(\ref{cond4}) consists of either a marginalization or a maximization over $\mathcal{Y}_n$ since $Y_n$ is a parent of $U_m$ by construction.
 The proof of Proposition~\ref{i-th} is in Appendix~\ref{proof:proof}.  Since the algorithm in Proposition \ref{i-th} consists of a backward inductive routine, its complexity is as in standard dynamic programming evaluation of influence diagrams \cite{Tatman1990}.

\begin{example}
To illustrate Proposition~\ref{i-th}, we follow the algorithm for the first three steps of the evaluation of the MID in Fig.~\ref{fig-ex}. Since the variable with the highest index, $Y_6$, is random, the backward induction procedure in Proposition~\ref{i-th} starts using the summation case of equation~(\ref{cond4}), specifically
\[
\bar{U}_6(\bm{y}_{B_6})=\bar{U}_{6,\mathbb{V}}(y_4,y_5)=\sum_{y_6\in\mathcal{Y}_6}k_3U_3(y_4,y_6)P(y_6\;|\;y_4,y_5).
\] 
Next the algorithm considers another random variable, $Y_5$. Since $5$ is the highest (and only) index in $P_2$, the backward induction is based on the summation in equation~(\ref{cond23}), which in this case equals
\[
\bar{U}_5(\bm{y}_{B_5})=\sum_{y_5\in\mathcal{Y}_5}\left(hk_2U_2(y_5)\bar{U}_6(\bm{y}_{B_6})+k_2U_2(y_5)+\bar{U}_6(\bm{y}_{B_6})\right)P(y_5\;|\;y_3,y_4).
\]
The backward induction has now reached $Y_4$, the first decision node. Although $Y_4$ is an argument of a utility function, it is not the highest index in $P_3$  and thus the algorithm uses equation~(\ref{cond1}) as
\[
\bar{U}_4(\bm{y}_{B_4})=\bar{U}_{4,\mathbb{D}}(y_3)=\max_{y_4\in\mathcal{Y}_4}\bar{U}_5(\bm{y}_{B_5}).
\]
\end{example}

We now arrange the EUs, that describe the effectiveness of the available decisions, in a vector as follows.
\begin{definition}
We define the \textbf{EU vector} $\bar{\bm{U}}_i$, $i\in[n]$, as
\begin{equation}
\label{veccond}
\bar{\bm{U}}_i=(\bar{U}_i(\bm{y}_{B_i}))_{ \bm{y}_{B_i}\in \bm{\mathcal{Y}}_{B_i}}^\T.
\end{equation}
\end{definition}

\subsection{Polynomial structure of expected utility}
\label{poli}
Generalizing work in \cite{Castillo1995,Darwiche2003}, we introduce a symbolic representation of both the probabilities and the utilities of an  MID. For $i\in \mathbb{V}$, $j\in[m]$, $y\in\mathcal{Y}_i$, $\pi\in\bm{\mathcal{Y}}_{\Pi_i}$ and  $\sigma\in \bm{\mathcal{Y}}_{P_j}$, we define the parameters 
\[
p_{iy\pi}=P(Y_i =y \;|\; \bm{Y}_{\Pi_i}=\pi) \qquad \text{ and } \qquad  \psi_{j\sigma}=U_j(\sigma).
\]
The first index of $p_{iy\pi}$ and $\psi_{j\sigma}$ refers to the random variable and utility vertex to which the parameter is related, respectively. The second index of $p_{iy\pi}$ relates to the state of the random variable, whilst the third one to the parents' instantiation. The second index of $\psi_{j\sigma}$ corresponds to the instantiation of the arguments of the utility function $U_j$. We take the indices within $\pi$ and $\sigma$ to be ordered from left to right in decreasing order, so that e.g. $p_{6101}$ for the diagram of Fig.~\ref{fig-ex} corresponds to $P(Y_6=1\;|\; Y_5=0, Y_4=1)$.  The \textit{probability} and \textit{utility vectors}  are given by  $\bm{p}_i=(p_{iy\pi})_{y\in\mathcal{Y}_i, \pi\in\bm{\mathcal{Y}}_{\Pi_i}}^\T$ and $\bm{\psi}_j=(\psi_{j\pi})_{\pi\in\bm{\mathcal{Y}}_{P_j}}^\T$, respectively. Parameters are listed within $\bm{p}_i$ and $\bm{\psi}_j$ according to a reverse lexicographic order over their indices~\cite{Cox2007a}\footnote{Let $\bm{\alpha},\bm{\beta}\in\mathbb{Z}^n$. We say that $\bm{\alpha}$ precedes $\bm{\beta}$ in reverse lexicographic order if the right-most non zero entry of $\bm{\alpha}-\bm{\beta}$ is positive.}. In contrast to ~\cite{Borgonovo2014}, we use different symbols for utilities and probabilities.  This is not only because these are formally different, but also because sensitivity methods can be tailored for these two types of indeterminates separately~\cite{Nielsen2003}.

\begin{example}
The symbolic parametrization of the MID in Fig.~\ref{fig-ex} is summarized in Table~\ref{para}. This is completed by the definition of the criterion weights $k_i$ and $h$ as in equation~(\ref{muia})-(\ref{k}). In Appendix ~\ref{b5} we report the symbolic definition of this MID using our Maple\textsuperscript{\textit{\tiny{TH}}} code. 
\end{example}

Because probabilities sum to one, for each $i$ and  $\pi$ one of the parameters $p_{iy\pi}$ can be written as one minus the sum of the others. 
Another constraint is induced by equation~(\ref{k}) on the criterion weights. However, unless otherwise indicated, we take  all the parameters to be  unconstrained.  Any unmodelled constraint can be added subsequently when investigating the geometric features of the \textit{admissible domains}~\cite{Nielsen2003},  i.e. regions of the parameters' space over which the preferred strategy does not change.

In the above parametrization, $\bar{\bm{U}}_i$ consists of a vector of polynomials  expressed in the unknown quantities $p_{ij\pi}$, $\psi_{j\sigma}$, $k_i$ and $h$, whose characteristics are specified  in  Theorem~\ref{polyexp}.
\begin{theorem}
\label{polyexp}
For an MID $G$ and $i\in [n]$, let $
c_i=\prod_{j\in B_i}r_j
$,  $U_l$ be the first utility node following $Y_i$ in the DS of $G$ and, for $l\leq j\leq m$,  $w_{ij}$ be the number of random nodes between $Y_i$ and $U_j$ (including $Y_i$)  in the DS of $G$. Then $\bar{\bm{U}}_i$ is a vector of dimension $c_i$ whose entries are  polynomials including, for $a=l,\dots,m$ and $b=l,\dots, a$, $r_{iba}$  monomials $m_{iba}$ of degree $d_{iba}$, where
\begin{equation}
\label{struct}
r_{iba}= \binom{a-l}{b-l}\prod_{j=i}^{j_a}r_j, \hspace{0.5cm}
d_{iba}= (b-l)+2(b-l+1)+w_{ia}, \hspace{0.5cm} m_{iba}=h^{b-l}m_{iba}',
\end{equation}
with $m_{iba}'$ a square-free monomial of degree $2(b-l+1)+w_{ia}$.
\end{theorem}

 \begin{table}
\renewcommand{\arraystretch}{1.5}
\begin{center}
\begin{tabular}{|l|}
\hline
$\bm{p}_2=(p_{211},p_{201},p_{210},p_{200})^{\T}$\\
$\bm{p}_3=(p_{3111},p_{3011},p_{3101},p_{3001},p_{3110},p_{3010},p_{3100},p_{3000})^{\T}$\\
$\bm{p}_5=(p_{5111},p_{5011},p_{5101},p_{5001},p_{5110},p_{5010},p_{5100},p_{5000})^{\T}$\\
$\bm{p}_6=(p_{6111},p_{6011},p_{6101},p_{6001},p_{6110},p_{6010},p_{6100},p_{6000})^{\T}$\\
$\bm{\psi}_1=(\psi_{11},\psi_{10})^{\T}$, $\bm{\psi}_2=(\psi_{21},\psi_{20})^{\T}$, $\bm{\psi}_3=(\psi_{311},\psi_{301}, \psi_{310},\psi_{300})^{\T}$\\
\hline
\end{tabular}
\end{center}
\caption{Parameterization associated to the MID in Fig.~\ref{fig-ex}. \label{para}}
\end{table}

The proof of Theorem~\ref{polyexp} is given in Appendix~\ref{appendix1}. Equation~(\ref{struct}) defines the \emph{structure} of the polynomials $\bar{\bm{U}}_i$ of the EU. Specifically, a polynomial is specified once its coefficients and its support (i.e. monomials which form the polynomial) are known. By
 structure of a polynomial  we mean the number of monomials in its support and the number of monomials having a certain degree (sum of exponents). 
 An algorithm for computing the polynomials in Theorem~\ref{polyexp} is presented in Section~\ref{algebra}, whose operations utilise the polynomial structure of EUs. 
 If the MID has one decision node only, then the entries of the EU vector correspond to the pieces defined in~\cite{Borgonovo2014}.  

\begin{example}
For the MID of Fig.~\ref{fig-ex} the polynomial structure of the entries of $\bar{\bm{U}}_5$ can be constructed as follows. From  $B_5=\{3,4\}$ it follows that $c_5=4$. Thus, $\bar{\bm{U}}_5$ is a column vector of dimension $4$. From $U_2\equiv U_l$ it follows that
\begin{equation*}
\begin{array}{cccccc}
r_{522}=2, & r_{523}=4, &r_{533}=4,&
d_{522}=3, &d_{523}=4, &d_{533}=7,
\end{array}
\end{equation*}
using the fact that $w_{52}=1$ and $w_{53}=2$. All monomials are square-free because the index $b$ of $r_{iba}$ in Theorem~\ref{polyexp} is either equal to $l$ or $l+1$. Each entry of $\bar{\bm{U}}_5$ is a square free polynomial of degree seven consisting of ten monomials: two  of degree $3$, four of degree $4$ and four of degree $7$. 
 \end{example}

 Since additive utility factorizations can be seen as special cases of multiplicative ones by setting $h = 0$, it follows that the EU polynomials of an additive ID are square-free.

\begin{corollary} \label{corAID}
In the notation of Theorem~\ref{polyexp}, the EU $\bar{\bm{U}}_i$, $i\in [n]$, of an additive ID $G$ is a vector of dimension $c_i$ whose entries are square free polynomials of degree $w_{im}+2$ including, for $a=l,\dots, m$, $r_{ia}$ monomials of degree $w_{ia}+2$, where
$r_{ia}=\prod_{j=i}^{j_a}r_j$.
\end{corollary}

\begin{proof}
This follows directly from Theorem~\ref{polyexp}, since an additive factorization can be derived by setting $n_I-1$, the exponent of $h$ in equation~(\ref{muia}), equal to zero. This
corresponds to fixing $b = l$ in Theorem~\ref{polyexp}.
  \end{proof}

So far we have assumed that the DM has not provided any numerical specification of the uncertainties and the values involved in the decision problem. This occurs for example if the system is defined through sample distributions of data from different experiments, where probabilities are only known with uncertainty.  But in practice sometimes  the DM is able to elicit the numerical values of some parameters. These numerical values can then simply be substituted to the corresponding probability and utility parameters in the system of polynomials constructed in Theorem~\ref{polyexp} employing e.g. a computer algebra system. In such a case the degree of the polynomials and possibly  the number of their monomials can decrease dramatically. We present in Section~\ref{sec:example} different plausible numerical specifications of  the parameters associated with the MID in Fig.~\ref{fig-ex}, and investigate how the outputs of the MID differ for the different quantifications.

\section{The symbolic algorithm}
\label{algebra}
In this section we develop an algorithm based on three operations which exploit the polynomial structure of EUs and use only linear algebra calculus.  The Maple\textsuperscript{\textit{\tiny{TH}}} code for their implementation is reported in Appendix~\ref{b3}\footnote{Some inputs of the Maple\textsuperscript{\tiny{\textit{TH}}}  functions in Appendix~\ref{b3} are different from those used in this section which are chosen to illustrate the procedure as concisely as possible.}. In contrast to other probabilistic symbolic algorithms (e.g.~\cite{Castillo1997a}), our procedure sequentially computes only monomials that are part of the EU polynomials and is thus much more efficient.
 
\subsection{A new algebra for MIDs}
\label{op}
 We need to introduce two  procedures entailing a change of dimension of probability, utility and EU vectors, named \texttt{EUDuplicationPsi} and \texttt{EUDuplicationP}. These are required in order to multiply parameters associated to compatible instantiations only, i.e. if the common conditioning variables associated to the parameters are instantiated to the same value. 
 
\begin{example}
In Algorithm~\ref{algo} we will need to compute the Schur (or element-wise) product $\circ$ between the probability vector $\bm{p}_6$ and the utility vector $\bm{\psi}_3$. However, as specified in Table~\ref{para}, $\bm{p}_6$ has length 8, whilst $\bm{\psi}_3$ has length 4. This is because $Y_5$ is a parent of $Y_6$ but not an argument of $U_3$. \texttt{EUDuplicationPsi} will then be needed to transform $\bm{\psi}_3$ to
\[\left(
\psi_{311},\psi_{301},\psi_{311},\psi_{301},\psi_{310},\psi_{300},\psi_{310},\psi_{300}
\right),
\] 
so that $\bm{p}_6\circ \bm{\psi}_3$ equals to
\begin{multline*}
\left(
\psi_{311}p_{6111},\psi_{301}p_{6011},\psi_{311}p_{6101},\psi_{301}p_{6001},\right.\\\left.\psi_{310}p_{6110},\psi_{300}p_{6010},\psi_{310}p_{6100},\psi_{300}p_{6000}
\right).
\end{multline*}
The above vector then only includes entries associated to compatible instantiations. 
\end{example}
 
 For conciseness, we detail here only \texttt{EUDuplicationPsi} and refer to Appendix~\ref{b2} for the code of both procedures. The steps of \texttt{EUDuplicationPsi} are shown in Algorithm~\ref{dupli}. For a vector $\bm{\psi}$, let $\bm{\psi}^{s,t}$ be the subvector of $\bm{\psi}$ including the entries from $s\cdot (t-1)+1$ to $s\cdot t$, for suitable $s,t\in\mathbb{Z}_{\geq 1}$. 
For $i\in[n{-1}]$ and $j\in[m]$, the procedure takes 7 elements as input: an EU $\bar{\bm{U}}_{i+1}$; the utility vector associated to the utility node preceding $Y_{i+1}$, $\bm{\psi}_j$; their dimensions, $c_{i+1}$ and $b_{j}$; the sets $B_{i+1}$ and $P_j$; the dimensions of all the probability vectors of the MID, $\bm{r}=(r_1,\dots, r_n)^\T$. 

\scalebox{0.85}{\parbox{\linewidth}{
\begin{spacing}{1.15}
\begin{pseudocode}[ruled]{EUDuplicationPsi}{\bar{\bm{U}}_{i+1}, \bm{\psi}_j, B_{i+1}, P_j,\bm{r}, c_{i+1}, b_j}
\FOR k \GETS i \DOWNTO 1  \DO \BEGIN 
\IF k\in \{\{B_{i+1}\cup P_{j}\}\setminus \{B_{i+1}\cap P_j\}\}\THEN \BEGIN
w_k=\prod_{l=k+1}^j \mathbbm{1}_{\{l\in \{B_{i+1}\cup  P_j\}\}}(r_l) \\
\IF k\in B_{i+1} \THEN \BEGIN
\bm{\psi}_j=\left(\begin{array}{ccc}
\underbrace{\begin{array}{ccc}
\bm{\psi}_j^{w_k,1}&\cdots &\bm{\psi}_j^{w_k,1}
\end{array}}_{\text{$r_k$ times}}
&\cdots&\underbrace{\begin{array}{ccc}
\bm{\psi}_j^{w_k,c_j/w_k}&\cdots &\bm{\psi}_j^{w_k,c_j/w_k}
\end{array}}_{\text{$r_k$ times}}
\end{array}\right) 
\END
\ELSEIF  k\in P_j  \THEN \BEGIN
\bar{\bm{U}}_{i+1}=\left(\begin{array}{ccc}
\underbrace{\begin{array}{ccc}
\bar{\bm{U}}_{i+1}^{w_k,1}&\cdots &\bar{\bm{U}}_{i+1}^{w_k,1}
\end{array}}_{\text{$r_k$ times}}
&\cdots&\underbrace{\begin{array}{ccc}
\bar{\bm{U}}_{i+1}^{w_k,c_i/w_k}&\cdots &\bar{\bm{U}}_{i+1}^{w_k,c_i/w_k}
\end{array}}_{\text{$r_k$ times}}
\end{array}\right)
\END 
\END
\END \\
\RETURN{ \bar{\bm{U}}_{i+1}, \bm{\psi}_j} 
\label{dupli}
\end{pseudocode}
\end{spacing}
} }

For all indices smaller than $i$ and not in $B_{i+1}\cap P_j$,  Algorithm~\ref{dupli} computes a positive integer number $w_k$ equal to the product of the dimension of the probability vectors with index bigger than $k$ belonging to $B_{i+1} \cup P_j$. The index $k$ is either in $B_{i+1}$ or in $P_j$. When $k\in B_{i+1}$, each block of $w_k$ rows of $\bm{\psi}_j$ is consecutively duplicated $r_k-1$ times.

The first of the three operations we introduce is \texttt{EUMultiSum}, which computes a weighted multilinear sum between a utility vector and an EU.  In the algorithm of Section~\ref{alg}, an \texttt{EUMultiSum} operation is associated to every utility vertex of the MID. This operation is required to formally assess the impact of a utility vertex to the overall EU and corresponds to a symbolic version of the sums in equation~(\ref{cond23}). Let  $P=\{P_1,\dots, P_m\}$. 

\begin{definition}[EUMultiSum]
\label{EUMultiSum}
For $i\in [n]$, let $\bar{\bm{U}}_{i+1}$ be an EU vector and $\bm{\psi}_j$ the utility vector of  node $U_j$, $j\in [m]$, succeeding $Y_i$ in the DS. The \texttt{EUMultiSum}, $+^{EU}$, between $\bar{\bm{U}}_{i+1}$ and $\bm{\psi}_j$ is defined as
\begin{enumerate}
\item $\bar{\bm{U}}_{i+1}',\bm{\psi}_j'\longleftarrow$\texttt{EUDuplicationPsi}($\bar{\bm{U}}_{i+1}$, $\bm{\psi}_j$, $B_{i+1}$, $P_j$, $\bm{r}$, $c_{i+1}$, $b_j$);
\item$h\cdot k_j\cdot(\bar{\bm{U}}_{i+1}'\;\circ\;\bm{\psi}_j')\;+\;k_j\cdot\bm{\psi}_j'\;+\;\bar{\bm{U}}_{i+1}'$, where $\circ$ and  $\cdot$ denote respectively the Schur (or element-wise) and the scalar products.
\end{enumerate}
\end{definition}
 
The second operation, \texttt{EUMarginalization} is applied to any random vertex of the MID. This operation is the symbolic equivalent of marginalizations (sums) $\sum_{y_i\in\mathcal{Y}_i}$ in Proposition~\ref{i-th}, often called variable elimination in the literature~\cite{Shachter1986}.
\begin{definition}[EUMarginalization]
\label{EU-Marginalization}
For $i\in\mathbb{V}$, let $\bar{\bm{U}}_{i+1}$ be an EU vector and $\bm{p}_i$ a probability vector. The \texttt{EUMarginalization}, $\Sigma^{EU}$, between $\bar{\bm{U}}_{i+1}$ and $\bm{p}_i$ is defined as
\begin{enumerate}
\item $\bar{\bm{U}}_{i+1}',\bm{p}_i'\longleftarrow$\texttt{EUDuplicationP}($\bar{\bm{U}}_{i+1}$, $\bm{p}_i$, $\Pi_i$, $P$, $\bm{r}$, $B_{i+1}$, $\mathbb{J}$);
\item $I_{i,\mathbb{V}}\times (\bar{\bm{U}}_{i+1}'\circ\bm{p}_i')$, where $\times$ is the standard matrix product and $I_{i,\mathbb{V}}$ is a matrix with $c_{i+1}s_i/r_i\in\mathbb{Z}_{\geq 1}$\footnote{This is so since $c_{i+1}=r_ia_{i+1}$, for an $a_{i+1}\in\mathbb{Z}_{\geq 1}$.} rows and $c_{i+1}s_i$ columns defined as 
\[
I_{i,\mathbb{V}}=\left(\begin{array}{cccc}
\left(
\begin{array}{cccc}
\bm{1}&\bm{0}&\cdots&\bm{0}
\end{array}
\right)
&\left(
\begin{array}{cccc}
\bm{0}&\bm{1}&\cdots&\bm{0}
\end{array}
\right)&
\cdots&
\left(
\begin{array}{cccc}
\bm{0}&\bm{0}&\cdots&\bm{1}
\end{array}
\right)
\end{array}
\right)^\T
\]
where $\bm{1}$ and $\bm{0}$ denote row vectors of dimension $r_i$ with all entries equal to one and zero respectively and $s_i=\prod_{k\in \{\Pi_i\setminus B_{i+1}\}}r_k$.
\end{enumerate}
\end{definition}

The last operation is a selection of a decision policy $y_i\in\mathcal{Y}_i$ in $\bar{\bm{U}}_{i+1}$, $i\in \mathbb{D}$, for every element of $\bm{\mathcal{Y}}_{\Pi(i)}$.

\begin{definition}[EUMaximization]
\label{EU-Maximization}
For $i\in \mathbb{D}$, let $\bar{\bm{U}}_{i+1}$ be an EU vector. An \texttt{EUMaximization} over $\mathcal{Y}_i$, $\max^{EU}_{\mathcal{Y}_i}$,  is defined by the following steps: 
 \begin{enumerate}
\item select a $y_i^*(\pi)\in \mathcal{Y}_i$, for $\pi\in\bm{\mathcal{Y}}_{\Pi(i)}$;
\item $I_{i,\mathbb{D}}\times \bar{\bm{U}}_{i+1}$, where  $I_{i,\mathbb{D}}$ is a matrix with $c_{i+1}/r_{i}\in\mathbb{Z}_{\geq 1}$ rows and $c_{i+1}$ columns  defined as
\[
I_{i,\mathbb{D}}=\left(\begin{array}{cccc}
\left(
\begin{array}{cccc}
\bm{e}_{\bm{y}_i^*(1)}&\bm{0}&\cdots&\bm{0}
\end{array}
\right)
&\left(
\begin{array}{cccc}
\bm{0}&\bm{e}_{\bm{y}^*_i(2)}&\cdots&\bm{0}
\end{array}
\right)&
\cdots&
\left(
\begin{array}{cccc}
\bm{0}&\bm{0}&\cdots&\bm{e}_{\bm{y}^*_i(c_{i+1}/r_{i})}
\end{array}
\right)
\end{array}
\right)^\T
\]
where $\bm{e}_{\bm{y}^*_i(\pi)}$, $\pi\in[ c_{i+1}/r_{i}]$, is a row vector of dimension $r_{i}$ whose entries are all zero but the one in position $y_i^*(\pi)$, which is equal to one.
\end{enumerate} 
\end{definition}
Using the terminology of  ~\cite{Bhattacharjya2010} and~\cite{Howard1968}, \texttt{EUMaximization} finds its natural application in \textit{open-loop} analyses, where one policy only is under scrutiny. In this case, the DM can simply fix the decision of interest and \texttt{EUMaximization} drops the polynomials associated to non-selected policies.\footnote{The Maple\textsuperscript{\tiny{\textit{TH}}} function \texttt{EUMaximization} in Appendix~\ref{b3} currently calls a subfunction \texttt{Maximize}, which randomly picks decisions. However, this can be modified to take into account a fixed policy given as input.}  Nevertheless, in \textit{closed-loop} analyses, where policies can vary, and in standard evaluation methods the first item of Definition~\ref{EU-Maximization} is critical for \texttt{EUMaximization}. It is not within the scope of this paper to present a methodology to identify EU maximizing decisions. However, within our symbolic approach polynomial optimization and semi-algebraic methods can be used to  guide the optimization process~\cite{parrilo2000}. In Section~\ref{sec:example} we present an example of the insights that the symbolic definition gives during the maximization step of an evaluation. 

Since all our operations simply consists of standard and matrix products, the  complexity of the algorithm for the symbolic computation of EUs we introduce below can be deduced by establishing the number of multiplications associated to each EU-operation. Formally, an \texttt{EUMultiSum} consists of $c_{i+1}s_i(2+m_{i+1})+1$ multiplications, where $m_{i+1}$ is the number of monomials in each entry of $\bar{\bm{U}}_{i+1}$ and can be deduced from Theorem \ref{polyexp}. An \texttt{EUMarginalization} consists of $c_{i+1}s_im_{i+1}+(c_{i+1}s_i)^2/r_i$ multiplications (without considering the sparsity of the matrix $I_{i,\mathbb{V}}$. Exploiting the structure of the matrix $I_{i,\mathbb{D}}$, an \texttt{EUMaximization} can be coded so that it does not perform any multiplication.

\subsection{Polynomial interpretation of the operations}
\label{polyopsec}
Each of the above three operations changes the EU vectors and their entries in a specific way we  formalize in Proposition~\ref{polyop}.

\begin{proposition}
\label{polyop}
For $i\in [n{-1}]$, let $\bar{\bm{U}}_{i+1}$ be an EU vector whose entries have the polynomial structure of equation~(\ref{struct}) and let $U_j$ be the vertex preceding $Y_{i+1}$ in the DS. Then in the notation of Theorem~\ref{polyexp}
\begin{itemize}
\item $\max^{EU}_{\mathcal{Y}_i}\bar{\bm{U}}_{i+1}$ has dimension $c_{i+1}/r_i\in\mathbb{Z}_{\geq 1}$ and its entries do not change polynomial structure;
\item $\bar{\bm{U}}_{i+1} +^{EU} \bm{\psi}_{j}$ has dimension $c_{i+1}t_{i}$, where $t_{i}=\prod_{k\in \{P_j\setminus B_{i+1}\}}r_k$, and each of its entries consists of $r_{(i+1)ba}$ monomials of degree $d_{(i+1)ba}$, $r_{(i+1)ba}$ monomials of degree $d_{(i+1)ba}+3$ and one monomial of degree 2;
\item $\bar{\bm{U}}_{i+1}\Sigma^{EU}\bm{p}_i$ has dimension $c_{i+1}s_i/r_i$, where $s_i=\prod_{k\in \{\Pi_i\setminus B_{i+1}\}}r_k$, and each of its entries consists of $r_ir_{(i+1)ba}$ monomials of degree $d_{(i+1)ba}+1$.
\end{itemize} 
\end{proposition}
This result directly follows from the definition of the above three operations. An illustration of Proposition~\ref{polyop} is given in Example~\ref{esempio} below.

\subsection{An algorithm for the computation of an MID's expected utilities} \label{alg}
The algorithm for the computation of an  MID's EUs is given in Algorithm~\ref{algo}. It receives as input the DS of the MID, $S$, the sets $\mathbb{J}$, $\mathbb{V}$ and $\mathbb{D}$, and the vectors $\bm{p}=(\bm{p}_1,\dots,\bm{p}_n)^\T$, $\bm{\psi}=(\bm{\psi}_1,\dots,\bm{\psi}_m)^\T$ and $\bm{k}=(k_1,\dots, k_m, h)^\T$. The algorithm corresponds to a symbolic version of the backward induction procedure working over the elements of the DS explicated in Proposition~\ref{i-th}. At each inductive step, a utility vertex is considered together with the variable that precedes it in the DS.

\begin{center} 
\scalebox{0.85}{\parbox{\linewidth}{
\begin{center}
\begin{spacing}{1.15}
\vspace{-0.5cm}
\begin{pseudocode}[ruled]{SymbolicExpectedUtility}{\mathbb{J},S,\bm{p},\bm{\psi}, \bm{k},\mathbb{V},\mathbb{D}}
\label{algo}
\bar{\bm{U}}_{n+1}=(0)\hspace{8.5cm}(1)\\
\FOR k \GETS n \DOWNTO 1 \hspace{7.015cm}(2)\DO \BEGIN 
\FOR l \GETS m \DOWNTO 1 \hspace{5.835cm}(3)\DO \BEGIN
\IF k=j_l \hspace{6.64cm} (4)\THEN \BEGIN
\IF k\in \mathbb{D} \hspace{5.17cm}(5)\THEN\BEGIN
\bar{\bm{U}}_{k}=\max^{EU}_{\mathcal{Y}_k}(\bar{\bm{U}}_{k+1}+^{EU}\bm{\psi}_l)\hspace{0.62cm}(6)
\END
\ELSE \BEGIN
\bar{\bm{U}}_k=\bm{p}_k\;\Sigma^{EU}\;(\bar{\bm{U}}_{k+1}+^{EU}\bm{\psi}_{l}) \hspace{0.55cm}(7)
\END
\END
\ELSEIF k\in \mathbb{D} \hspace{5.67cm}(8)\THEN \BEGIN
\bar{\bm{U}}_k=\max^{EU}_{\mathcal{Y}_k}\bar{\bm{U}}_{k+1}\hspace{3.59cm}(9)
\END
\ELSE \BEGIN
\bar{\bm{U}}_k=\bm{p}_k\;\Sigma^{EU}_{\mathcal{Y}_k}\;\bar{\bm{U}}_{k+1}\hspace{3.41cm}(10)
\END
\END
\END\\
\RETURN{ \bar{\bm{U}}_{1}} \hspace{8.12cm} (11)
\label{algo1}
\end{pseudocode}
\vspace{-0.5cm}
\end{spacing}
\end{center}
}}\\
\end{center}

In line (1) the EU $\bar{\bm{U}}_{n+1}$ is initialized to $(0)$, namely a vector of dimension one including a zero.   
Lines (2) and (3) index a reverse loop over the indices of the variables and the utility vertices respectively (starting from $n$ and $m$). If the current index corresponds to a variable preceding a utility vertex in the DS (line 4), then the algorithm jumps to lines (5)-(7). Otherwise it jumps to lines (8)-(10). In the former case, the algorithm computes, depending on whether or not the variable is controlled (line 5), either an \texttt{EUMaximization} over $\mathcal{Y}_k$ (line 6) or an \texttt{EUMarginalization} (line 7) with $\bm{p}_k$,  jointly to an \texttt{EUMultiSum} with $\bm{\psi}_l$. In the other case, \texttt{EUMaximization} and \texttt{EUMarginalization} operations are performed without \texttt{EUMultiSum}. The Maple\textsuperscript{\tiny{\textit{TH}}}
function \texttt{SymbolicExpectedUtility}  in Appendix~\ref{b4} is an implementation of  Algorithm~\ref{algo}.

\begin{example}
\label{esempio}
For the MID in Fig.~\ref{fig-ex} the \texttt{SymbolicExpectedUtility} function first considers  the random vertex $Y_6$ which precedes the utility vertex $U_3$ and therefore first calls the \texttt{EUMultiSum} function. For this  MID
\[
\begin{array}{llll}
 P_3=\{4,6\}, & t_6=4,  &
 \Pi_6=\{4,5\}, & s_6=2.
\end{array}
\]
Thus, first $\bar{\bm{U}}_7$ is replicated four times (since $t_6=4$) via \texttt{EUDuplicationPsi} and 
\begin{equation}
\label{ms3}
\bar{\bm{U}}_7+^{EU}\bm{\psi}_3=\left(
\begin{array}{cccc}
k_3\psi_{11}&
k_3\psi_{01}&
k_3\psi_{10}&
k_{3}\psi_{00}
\end{array}
\right)^\T.
\end{equation}
Then, the rhs of equation~(\ref{ms3}) is duplicated via \texttt{EUDuplicationP} (as $s_6=2$) and  
\begin{equation}
\label{eq:ciao}
\bar{\bm{U}}_6=I_{6,\mathbb{V}} \times \bar{\bm{U}}_6'\circ \bm{p}_6= \left(k_3\psi_{31j}p_{61ij}+k_3\psi_{30j}p_{60ij}\right)^\T_{i,j=0,1},
\end{equation}
where $\bar{\bm{U}}_6'$ is equal to the duplicated version of the rhs of equation~(\ref{ms3}).
The vector $\bar{\bm{U}}_6$ has dimension four and its entries include two monomials of degree $3$. Since the random vertex $Y_5$ is the unique parent of $U_2$ the \texttt{SymbolicExpectedUtility} function follows the same steps as before. \texttt{EUMultiSum} is called and
\begin{equation}
\label{u6p}
\bar{\bm{U}}_5'\triangleq\bar{\bm{U}}_6+^{EU}\bm{\psi}_2= (h\cdot\bar{\bm{U}}_6+1)\cdot k_2\circ \left(
\begin{array}{cccc}
\psi_{21}&
\psi_{20}&
\psi_{21}&
\psi_{20}
\end{array}
\right)^\T+\bar{\bm{U}}_6   .   
\end{equation}
The polynomial $\bar{\bm{U}}_5'$ is the sum of two monomials of degree $3$ inherited from $\bar{\bm{U}}_6$, of two monomials of degree $6$ (from the first term on the rhs of equation~(\ref{u6p})) and one monomial of degree $2$ (from the last term on the rhs of equation~(\ref{u6p})). Its dimension is equal to four since $c_6=4$ and $s_5=0$ (i.e. no \texttt{EUDuplicationPsi} is required). 
Thus, \texttt{EUMultiSum} manipulates the EU vector according to Proposition~\ref{polyop}. The \texttt{EUMarginalization} function computes
$
\bar{\bm{U}}_5=
I_{5,\mathbb{V}}\times
((\begin{array}{cc}
\bar{\bm{U}}_5'&
\bar{\bm{U}}_5'
\end{array}
)^\T\circ
\bm{p}_5)
$. Each entry of $\bar{\bm{U}}_5$ has twice the number of monomials of the entries of $\bar{\bm{U}}_5'$ and each  monomial of $\bar{\bm{U}}_5$ has degree $d+1$, where $d$ is the degree of each monomial of $\bar{\bm{U}}_5'$ (whose entries are homogeneous polynomials). These vectors also have the same dimension since $t_5=2$ and $r_5=2$. Thus, this \texttt{EUMarginalization} changes the EU vector according to Proposition~\ref{polyop}. The entry $\bar{U}_5(y_3,y_4)$, with $y_3,y_4=0,1$, of this EU  can be shown to be equal to the sum of the terms in Table~\ref{table}.

\begin{table}
\renewcommand{\arraystretch}{1.25}
\begin{center}
\resizebox{\columnwidth}{!}{
\begin{tabular}{|c|}
\hline
$k_2(\psi_{21}p_{51y_4y_3}+\psi_{20}p_{50y_4y_3})$\\
$k_3(\psi_{31y_4}p_{611y_4}+\psi_{30y_4}p_{601y_4})p_{51y_4y_3}+k_3(\psi_{31y_4}p_{610y_4}+\psi_{30y_4}p_{600y_4})p_{50y_4y_3}$\\
$hk_2k_3((\psi_{31y_4}p_{610y_4}+\psi_{30y_4}p_{600y_4})\psi_{20}p_{50y_4y_3}+(\psi_{31y_4}p_{611y_4}+\psi_{30y_4}p_{601y_4})\psi_{21}p_{51y_4y_3})$\\
\hline
\end{tabular}}
\end{center}
\caption{The utility funtion $\bar{\bm{U}}_5$ is the sum of the three polynomials in this table. \label{table}}
\end{table}

The algorithm then considers the controlled variable $Y_4$. Since $4\not\in \mathbb{J}$, $Y_4$ is not the argument of a utility function with the highest index and therefore the algorithm calls the \texttt{EUMaximization} function. Suppose the DM decides to fix $Y_4=1$ when $Y_3=1$ and $Y_4=0$ when $Y_3=0$.  Then \texttt{EUMaximization} returns $\bar{\bm{U}}_4=I_{4,\mathbb{D}}\times \bar{\bm{U}}_5$, where $I_{4,\mathbb{D}}$ is a $2\times 4$ matrix with ones in positions $(1,1)$ and $(2,4)$ and zeros otherwise.
Proposition~\ref{polyop} is respected since the entries of $\bar{\bm{U}}_4$ have the same polynomial structure of those of $\bar{\bm{U}}_5$ and $\bar{\bm{U}}_4$ has dimension $2$.

The \texttt{SymbolicExpectedUtility} function then applies in sequence the operations defined in  Section~\ref{op}. For the MID in Fig.~\ref{fig-ex} this sequentially computes the following quantities, assuming the DM fixed $Y_1=1$
\[
\begin{array}{lll}
\bar{\bm{U}}_3'=h\cdot k_1\cdot \bar{\bm{U}}_4\circ\bm{\psi}_1
+\bar{\bm{U}}_4+ k_1\cdot 
\bm{\psi}_1,
&& \bar{\bm{U}}_3 = I_{3,\mathbb{V}}\times \left(\left(
\begin{array}{cccc}
\bar{\bm{U}}_3'&
\bar{\bm{U}}_3'&
\bar{\bm{U}}_3'&
\bar{\bm{U}}_3'
\end{array}
\right)^\T
\circ\bm{p}_3\right),\\
\bar{\bm{U}}_2=I_{2,\mathbb{V}}\times
\left(\bar{\bm{U}}_3\circ
\bm{p}_2
\right),
&&\bar{\bm{U}}_1=\left(
\begin{array}{cc}
1&0
\end{array}
\right)\times 
\bar{\bm{U}}_2  .   
\end{array}
\]
\end{example}

The overall complexity of the algorithm can be formally deduced by counting the number of multiplications it involves. Given the number of such products for each of our EUoperations, the overall number of operations of Algorithm 4.2 is the sum of the multiplications of its operations and will depend on the topology of the ID network. An empirical study of the efficiency of our implementation is given below.
  
\subsection{Simulation study} \label{sect:simulation}
\begin{table}
\begin{center}
\begin{tabular}{|c|c|c|c|c|c|c|}
\hline
Net. & Free par. & \# $\mathbb{V}$ &\# $\mathbb{D}$ & $m$& \# $E(G)^*$ & Avg. indegree\\
\hline
A & 44 &4 & 2 &3 &10 &1.556\\
\hline 
B&96 & 6 & 5 & 5& 23 &1.875\\
\hline
C & 192 & 9 & 6 & 6 &33 &2.286\\
\hline
D & 252 & 13 & 7 & 8 & 46 & 2.429\\
\hline
E & 356 & 17 & 8 & 9 & 62 & 2.412\\
\hline
\end{tabular}
\end{center}
\caption{Summaries of the IDs considered in the simulation study: Net. - ID identifier; Free par. - number of free parameters; \# $\mathbb{V}$ - number of random nodes; \# $\mathbb{D}$ - number of decision nodes; $m$ - number of utility nodes; \# $E(G)^*$ - number of edges without those into decision nodes; Avg. indegree - average number of edges directed into vertices (without decision nodes). \label{tablesimul}}
\end{table}

To investigate the complexity of the symbolic algorithm in Section~\ref{alg}, we perform a simulation study comprising of 5 IDs, whose features are summarized in Table~\ref{tablesimul}  all with binary variables. We first produced a full symbolic definition of utilities and probabilities and then run our symbolic algorithm for both multiplicative and additive utility factorizations in Maple\textsuperscript{\textit{\tiny{TH}}}. This gives as output the EU vectors $\bar{U}_i$ associated to every random and decision nodes of the IDs. We also built the same networks using the GeNIe Modeler software of \lq\lq{B}ayesFusion, LLC\rq\rq (freely available for academics at \textit{http://www.bayesfusion.com}), which embeds numerical evaluation techniques for IDs. After building the networks in GeNIe, we specified numerical values for the probabilities and utilities and then ran the evaluation algorithm. It is important to highlight that GeNIe considers only additive factorizations between utility nodes.

The results of the study are summarized in Table~\ref{study}. Whilst the memory allocated in Maple\textsuperscript{\textit{\tiny{TH}}} is almost identical for IDs with multiplicative and additive utility factorizations, the computation time as well as the number of monomials is much larger for MIDs. Comparing the computation times of GeNIe with those in Maple\textsuperscript{\textit{\tiny{TH}}}, we notice that whilst these are of the same magnitude for smaller IDs, for larger networks GeNIe appears to become significantly slower.  However we underline that the two softwares produce different outputs: expected utility vectors with polynomial entries in Maple\textsuperscript{\textit{\tiny{TH}}} and numerical evaluation of the ID in GeNIe.

\begin{table}
\begin{center}
\begin{tabular}{|c|c|c|c|c|c|c|c|}
\hline
&\multicolumn{3}{ |c| }{Multiplicative}& \multicolumn{3}{ |c| }{Additive} & GeNIe  \\
\hline
Net.& Mem. All. & Time & \# Mon. & Mem. All. & Time & \# Mon. & Time\\
\hline
A& 116KiB & 3ms & 84 & 114KiB&2ms & 28 &9.5ms\\
\hline
B& 0.95MiB& 36ms&1426 & 0.94MiB& 29ms &138&13ms\\
\hline
C&1.05MiB&59ms&17810&1.04MiB&36ms&650&30ms\\
\hline
D&1.25MiB&1.45s&1590674&1.23MiB&82ms&17034&3.6s\\
\hline
E&1.41MiB&38.5s&$>10^9$&1.38MiB&355ms&148106&81.2s\\
\hline
\end{tabular}
\end{center}
\caption{Complexity summaries of the symbolic algorithms in Maple\textsuperscript{\textit{\tiny{TH}}} and of the numerical algorithms in GeNIe: Mem. All. - memory allocation; Time - computation time; \# Mon. - number of monomials of the final EU vectors $\bar{U}_1$. \label{study}}
\end{table}

Although the efficiency of the symbolic algorithms highly depends on the size of the network, 
 the  simulation study in this section shows that even with the current capabilities of general-purpose computer algebra softwares symbolic techniques in decision making problems of medium/large scale are usefully applicable. In particular for IDs embedding additive factorizations, computation times increase at a slower pace than in the other cases (GeNIe and MIDs) and could thus be efficiently implemented in much larger domains than those presented here. However, it is uncommon to perform sensitivity studies over networks much larger than those investigated here. We refer a discussion of the handling of massive networks in our symbolic framework to Section \ref{discussion}.

\section{Modifying the topology of the MID}
Algorithm~\ref{algo}  works under the assumption that the MID is in extensive form  whose importance was discussed in Section~\ref{subsec:evalMIDS}.
It has been recognized that typically a DM will build an MID so that variables and decisions are ordered in the way they actually 
happen and  this might not correspond to the order in which variables are observed. Thus, MIDs often are not in extensive form.
But it is always possible to transform an MID into one in extensive form, although this might entail the loss of conditional independence structure. 
In Section~\ref{arcbarren} we consider 
two of the most common operations that can do this: 
edge reversal and barren node elimination. 

In practice  DMs often also include in the MID variables that subsequently turn out  not to be strictly necessary for identifying an optimal policy.  DMs are able to provide probabilistic judgements for conditional probability tables associated to an MID with variables describing the way they understand the unfolding of events. However their understanding usually includes variables that are redundant for the evaluation of the MID. In Section~\ref{suffy} we describe the polynomial interpretation of a criterion introduced in~\cite{smith89a,smith89} to identify a subgraph of the original MID whose associated optimal decision rule is the same as the one of the original MID. 

\label{topology}

\subsection{Rules to transform an MID in extensive form}
\label{arcbarren}

\begin{figure}
\begin{center}
\scalebox{0.7}{
\xymatrix{
U_1&*+[Fo]{Y_3}\ar[l]\ar[dd]\ar[dr]&*+[Fo]{Y_2}\ar[l]&*+[F]{Y_1}\ar[dl]\ar[l]\ar@/^1pc/[ll]\\
&&*+[F]{Y_4}\ar[d]\ar[dl]\ar[rd]&\\
U_2&*+[Fo]{Y_5}\ar[r]\ar[l]&*+[Fo]{Y_6}\ar[r]&U_3
}}\hspace{0.5cm}
\scalebox{0.7}{\xymatrix{
U_1&*+[Fo]{Y_3}\ar[l]\ar[dd]\ar[dr]\ar[r]&*+[Fo]{Y_2}&*+[F]{Y_1}\ar[dl]\ar[l]\ar@/^1pc/[ll]\\
&&*+[F]{Y_4}\ar[d]\ar[dl]\ar[rd]&\\
U_2&*+[Fo]{Y_5}\ar[r]\ar[l]&*+[Fo]{Y_6}\ar[r]&U_3
}}\hspace{0.5cm}
\scalebox{0.7}{\xymatrix{
U_1&*+[Fo]{Y_3}\ar[l]\ar[dd]\ar[dr]&&*+[F]{Y_1}\ar[dl]\ar@/^1pc/[ll]\\
&&*+[F]{Y_4}\ar[d]\ar[dl]\ar[rd]&\\
U_2&*+[Fo]{Y_5}\ar[r]\ar[l]&*+[Fo]{Y_6}\ar[r]&U_3
}}
\caption{Example of a sequence of manipulations of a non extensive form MID. }\label{fig:1}
\end{center}
\end{figure}
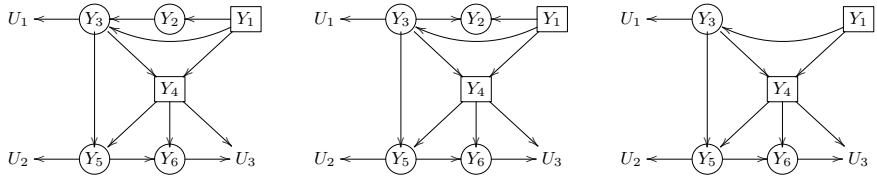

The two operations of \textbf{arc reversal} and \textbf{barren node removal} are often used in combination by first reversing the direction of some edges of the MID and then removing vertices that, consequently to the reversals, becomes barren, i.e. have no children~\cite{Shachter1986}. 

\begin{example}
The MID on the left of Fig.~\ref{fig:1} is
a non-extensive variant of the MID in Fig.~\ref{fig-ex} not including the edge $(Y_2,Y_4)$. 
The  MID in the centre of Fig.~\ref{fig:1} is obtained by the reversal of the edge $(Y_2,Y_3)$ and the  MID on the right of Fig.~\ref{fig:1} is the network in extensive form obtained by deleting the barren node $Y_2$. 
\end{example}

First we introduce a terminology to characterize a special pair of parent/child as in~\cite{Castillobook} for which edge reversals are simpler~\cite{Jensenbook}.  
 It is not the purpose of this paper to identify an optimal sequence of edge reversals, i.e. one yielding a simplified MID with the least number of vertices and edges. 
Instead  we can use algorithms already devised for standard IDs to perform diagram transformations~\cite{Shachter1990} by arc reversal and barren node removals. 
 We say that $Y_i$  is the father of $Y_j$ and $Y_j$ its son if the edge set of the MID includes $(Y_i, Y_j)$ and there is no other path starting at $Y_i$ and terminating at $Y_j$ that connects them. 

\begin{example}
For the MIDs in Figure \ref{fig:1}, both $Y_4$ and $Y_5$ are parents of $Y_6$, but only $Y_5$ is its father since there is the path $(Y_4,Y_5,Y_6)$. Notice that a vertex can have only one father but more than one son. 
\end{example}

\begin{proposition}
\label{manipulation}
The evaluation of an MID $G$ provides the same optimal policy as the MID  $G'$ obtained by implementing any of the following manipulations:
\begin{itemize}
\item \textbf{Arc Reversal}: for $i,j\in\mathbb{V}$, if $Y_i$ is the father of $Y_j$ in $G$ reverse the arc $(Y_i,Y_j)$ into $(Y_j,Y_i)$ and change the edge set as 
\[
E(G')=E(G)\setminus \{(Y_i,Y_j)\}\cup  \{(Y_k,Y_i): k\in\{\Pi_j\cup j\}\setminus i\}\cup\{(Y_k,Y_j): k\in \Pi_i\},
\]
\item \textbf{Barren Node Removal}: for $i\in\mathbb{V}$, remove the vertex $Y_i$ if this has no children and transform the diagram according to the following rules:
\[
V(G')=V(G)\setminus \{Y_i\},\;\;\;\;\hspace{0.5cm}E(G')=E(G)\setminus \big\{(Y_k,Y_i): k\in \Pi_i\big\}.
\]
\end{itemize}
\end{proposition}

Arc reversal and barren node removal change the symbolic parametrization of the MID according to Proposition~\ref{par}. 
After an arc reversal, the diagram $G'$ includes the edge $(Y_j,Y_i)$ where $i<j$. Algorithm~\ref{algo}, and similarly the  Maple\textsuperscript{\textit{\tiny{TH}}} function \texttt{SymbolicExpectedUtility}, works through a backward induction  over the indices of the variables and, by construction, always either marginalize or maximize a vertex before its parents. It cannot therefore be applied straightforwardly to the diagram $G'$. We define here the adjusted Algorithm~\ref{algo} which takes into account the reversal of an arc by, roughly speaking, switching the order in which the  variables associated to the reversed edge are marginalized during the procedure.  Specifically, in the adjusted Algorithm~\ref{algo} a marginalization operation is performed over $Y_i$ at the $n-j+1$ backward inductive step, whilst for $Y_j$ this happens at the $n-i+1$ step. Therefore $\bar{\bm{U}}_j'$ is the EU associated to $G'$ after the marginalization of $Y_i$ and  $\bar{\bm{U}}_i'$ is the EU after the marginalization of $Y_j$. Note that under this operation the sets $\mathbb{J}$ and $B_i$, $i\in [n]$, might change: we respectively call $\mathbb{J}'$ and $B'_i$ the ones associated to $G'$.

\begin{proposition}
\label{par}
Under the conditions of Proposition~\ref{manipulation}, 
let $p_{iy\pi}'$ and $\Pi'_i$  be a parameter and a parent set associated to the diagram $G'$ resulting from arc reversal and barren node removal:
\begin{itemize}
\item for $i,j\in\mathbb{V}$, if $\Pi'_i$ and $\Pi'_j$ are the parent sets of $Y_i$ and $Y_j$ after the reversal of the edge $(Y_i,Y_j)$, then 
the parametrization associated to $G'$ is
\begin{equation*}
p_{iy_i\pi'_i}'= \frac{p_{jy_j\pi_j}p_{iy_i\pi_i}}{\sum_{y_i\in\mathcal{Y}_i}p_{jy_j\pi_j}p_{iy_i\pi_i}},
\hspace{1cm}p_{jy_j\pi'_j}'=\sum_{y_i\in\mathcal{Y}_i}p_{jy_j\pi_j}p_{iy_i\pi_i},
\end{equation*}
for $\pi_i\in\bm{\mathcal{Y}}_{\Pi_i}$, $\pi_j\in\bm{\mathcal{Y}}_{\Pi_j}$, $\pi'_i\in\bm{\mathcal{Y}}_{\Pi'_i}$,$\pi'_j\in\bm{\mathcal{Y}}_{\Pi'_j}$, $y_i\in\mathcal{Y}_i$ and $y_j\in\mathcal{Y}_j$;
\item for $i,j\in \mathbb{V}$, assume that after the reversal of the edges $(Y_i,Y_j)$, for every children $Y_j$ of $Y_i$, $Y_i$ is now a barren node and let $\Pi_{j\setminus i}=\Pi_j\setminus \{i\}$. Then
\begin{itemize}
\item in the new parametrization $\bm{p}_i'$ is deleted;
\item in the original parametrization $\bm{p}_i$ is deleted and $p_{jy_j\pi_{j\setminus i}0}=\cdots=p_{jy_j\pi_{j\setminus i}r_i-1},$ for $y_j\in\mathcal{Y}_j$, $\pi_{j\setminus i}\in\bm{\mathcal{Y}}_{\Pi_{j\setminus i}}$, where the fourth index of $p_{jy_j\pi_{j\setminus i}i}$, $i\in[r_i{-1}]$, refers to the instantiation of $Y_i$.
\end{itemize}
\end{itemize}
\end{proposition}
 The proof of this proposition is reported in Appendix~\ref{proofteo2}.

\begin{example}
Reversing the edge $(Y_2,Y_3)$ in the MID on the left of Fig.~\ref{fig:1}, by Proposition~\ref{par} 
we obtain:
\begin{equation*}
p_{3y_3y_1}'=p_{3y_31y_1}p_{21y_1}+p_{3y_30y_1} p_{20y_1},\hspace{1cm}p_{2y_2y_3y_1}'=\frac{p_{3y_3y_2y_1}p_{2y_2y_1}}{p_{3y_31y_1}p_{21y_1}+p_{3y_30y_1}p_{20y_1}}.
\end{equation*}
for $y_1,y_2,y_3\in\{0,1\}$. Proposition~\ref{par} also specifies that the deletion of the vertex $Y_2$ as on the right of Fig.~\ref{fig:1} simply corresponds to cancelling the vectors $\bm{p}_2$ and $\bm{p}_2'$ and setting $p_{3y_31y_1}$ equal to $p_{3y_30y_1}$ for any $y_1,y_3\in\{0,1\}$.
\end{example}
 
 Note that arc reversals, just as posterior probabilities in symbolic inferences, transform EUs into rational functions of multilinear polynomials. However, Proposition \ref{par} suggests a straightforward model's reparametrization, which maps  EUs back to polynomial functions. In addition,  Proposition~\ref{par} shows that manipulations of the diagram change the polynomial structure of the EUs under the new parametrization $\bm{p}'$ that we formally study in Lemmas \ref{polarc} and \ref{br} below. We assume here for simplicity that $i\not \in P_j$, $j\in [m]$. There is no loss of generality in this assumptions since arguments of utility functions cannot be deleted from the diagram without changing the result of the evaluation.

\begin{lemma}
\label{polarc}
Under the assumptions of Proposition~\ref{par}  and in the notation of Theorem~\ref{polyexp}, suppose we reverse the arc $(Y_i,Y_j)$ in an MID $G$. Let $x$ be the smallest index in $\Pi_i\cup \Pi_j$. Evaluating $G$ using the adjusted Algorithm~\ref{algo} the following holds:
	\begin{enumerate} 
	\item if $j\not \in \mathbb{J}$, then
		\begin{itemize}
		\item the entries of $\bar{\bm{U}}_j'$ have $r_ir_{jba}/r_j\in\mathbb{Z}_{\geq 1}$\footnote{This is so since $r_{jba}=r'r_j$ for some $r'\in\mathbb{Z}_{\geq 1}$.} monomials of degree $d_{jba}$; for $i<k<j$, the entries of $\bar{\bm{U}}_k'$ can have different polynomial structure from the ones of $\bar{\bm{U}}_k$ according to Proposition~\ref{polyop};
		\item the vectors $\bar{\bm{U}}_k'$, $x<k\leq j$, have dimension $c_k=\prod_{s\in C_k\setminus\{k,\dots,n\}}r_s$  where 
		         $ C_k=B_k\cup\{l: (Y_l,Y_i) \mbox{ or } (Y_l,Y_j) \in E(G')\}$;
		\end{itemize}
	\item \label{1b} if $j\in \mathbb{J}\cap \mathbb{J}'$, then
		\begin{itemize}
		\item the entries of $\bar{\bm{U}}_j'$ have $r_ir_{(j+1)ba}$ monomials of degree $d_{(j+1)ba}+1$ and, for $i<k<j$, the entries of $\bar{\bm{U}}_k'$ have a different polynomial structure from the ones of $\bar{\bm{U}}_k$ according to Proposition~\ref{polyop};
		\item for $x<k<j$, $\bar{\bm{U}}_k'$ has dimension  $ a_k=\prod_{s\in \{A_k\setminus \{k,\dots,n\}\}}r_s$, with $A_k=C_k\cup P_{j_j}$;
		\end{itemize}
	\item if $j\not\in \mathbb{J}'$, suppose $j\in P_t$ and  $s$ is the second highest index in $P_t$, then
		\begin{itemize}
		\item for $s<k\leq j$, the entries of  $\bar{\bm{U}}_k'$ have the polynomial structure specified in point 2 and dimension $a_k$; 
		\item $i< k\leq s$,  the entries of $\bar{\bm{U}}_k'$ have the polynomial structure specified in point 1 and dimension $c_k$. 
		\item for $x<k\leq i$, $\bar{\bm{U}}_k'$ has dimension $c_k$ and the polynomial structure of its entries does not change;
		\end{itemize}
	\end{enumerate}
\end{lemma}
The proof of this lemma is provided in Appendix~\ref{corollario}.

We next consider how a barren node removal changes the EU vectors.
\begin{lemma}
\label{br}
In the notation of Lemma \ref{polarc}, let $Y_z$ be the child of $Y_i$ in $G$ with the highest index and remove the barren node $Y_i$ in $G'$. Then
		\begin{itemize}
		\item  for $i<k\leq z $, $\bar{\bm{U}}_k'$ has  $c_k/r_i$ entries whose polynomial structure does not change;
		\item for $k\leq i$, $\bar{\bm{U}}_k'$ has dimension $c_k$ and its entries have $r_{kba}/r_i$ monomials of degree $d_{kba}-1$.
         	\end{itemize}
\end{lemma}
The proof of this lemma is provided in Appendix~\ref{corollario}.

\begin{example}
After the reversal of the edge $(Y_2,Y_3)$ from the network on the left of Fig.~\ref{fig:1}, the polynomial structure of the EUs associated to the original and manipulated diagrams is reported in Table~\ref{ciao} by $\bar{\bm{U}}_i$ and $\bar{\bm{U}}_i^r$ respectively. 
Since $Y_3$ is the only argument of $U_1$ we are in Item~(\ref{1b}) of Lemma~\ref{polarc}. The EU $\bar{\bm{U}}^r_3$ is obtained running the adjusted Algorithm~\ref{algo} over the graph in the centre of Fig.~\ref{fig:1} after the marginalization of $Y_2$. This can be noted to  change according to Lemma~\ref{polarc}, by comparing its structure to the one of $\bar{\bm{U}}_4$. Furthermore, $\bar{\bm{U}}_2^r$ and $\bar{\bm{U}}_1^r$ have the same polynomial structures as  $\bar{\bm{U}}_2$ and $\bar{\bm{U}}_1$. 
The last $3$ columns of the Table~\ref{ciao} show the polynomial structure of the EUs $\bar{\bm{U}}_3^b$ associated to the MID on the right of Fig.~\ref{fig:1} which does not include $Y_2$. 
According to Lemma~\ref{polarc}, $\bar{\bm{U}}_3^b$ has the same polynomial structure of $\bar{\bm{U}}_3$
and for each row of the table, the number of monomials with degree $d$ in $\bar{\bm{U}}_1^b$ is half the number of monomials of $\bar{\bm{U}}_1$ having degree $d+1$. 
\end{example}

\subsection{The sufficiency principle}
\label{suffy}
After an  MID has been transformed in extensive form according to the rules in Section~\ref{arcbarren}, further manipulations can be applied to simplify its evaluation, such as the sufficiency principle, which
mirrors the concept of sufficiency in statistics and is based on the concept of d-separation for DAGs~\cite{Pearl1988} formally defined below. 

We first introduce a few concepts from graph theory. The moralized graph $G^M$ of the MID $G$ is a graph with the same vertex set  of $G$. Its directed edges include the directed edges of $G$ and 
an undirected edge between any two vertices which are not joined by an edge in $G$ but which are parents of the same child in $Y_i$, $i\in \mathbb{V}$. The skeleton of $G^M$, $\mathcal{S}(G^M)$, is
a graph with the same vertex set of $G^M$ and an undirected edge between any two vertices $(Y_i,Y_j)\in V(G^M)$ if and only if there is a directed or undirected edge between $Y_i$
and $Y_j$ in $G^M$. 

\begin{definition}
For any three disjoint subvectors $\bm{Y}_A, \bm{Y}_B, \bm{Y}_C \in V(G^M)$, $\bm{Y}_A$ is \textit{d-separated} from $\bm{Y}_C$ by $\bm{Y}_B$ in the moralized graph $G^M$ of an MID $G$ if and only if any path from any vertex $Y_a\in\bm{Y}_A$ to any vertex $Y_c\in\bm{Y}_C$ passes through a vertex $Y_b\in\bm{Y}_B$ in its skeleton $\mathcal{S}(G^M)$.
\end{definition}
   
\begin{proposition}
\label{suff}
Let $j\in \mathbb{D}$, $i\in \mathbb{V}\,\cap\, \Pi_j$ and $Ch_i$ be the index set of the children of $Y_i$. Then if $Y_i$ is d-separated from $\{U_k, \mbox{ for } k \; s.t. \; i\leq j_k\}$ by $\{Y_k: k\in \{\Pi_j\setminus i\}\}\cup \{Y_k:k\in\mathbb{D}\}$ in the MID $G$, the \emph{sufficiency principle} guarantees that the evaluation of the graph $G'$ provides the same optimal policy as $G$, where $G'$ is such that $V(G')=V(G)\setminus \{Y_i\}$ and $E(G')$ is equal to
\begin{multline*}
E(G)\setminus \{(Y_i,Y_j): j\in Ch_i\}\setminus \{(Y_k,Y_i): k\in \Pi_i\}
\cup \{(Y_k,Y_j): j\in Ch_i, k \in \Pi_i\}.
\end{multline*}
\end{proposition}
\begin{table}
\begin{center}
\def\arraystretch{1.2}
\resizebox{\columnwidth}{!}{
\begin{tabular}{|c|c|c|c|c|c|c|c|c|c|c|c|c|c|c|c|c|c|}
\hline
 \multicolumn{3}{|c|}{$\bar{\bm{U}}_2\equiv \bar{\bm{U}}_1$}&\multicolumn{3}{|c|}{$\bar{\bm{U}}_3$}&\multicolumn{3}{|c|}{$\bar{\bm{U}}_4$}&\multicolumn{3}{|c|}{$\bar{\bm{U}}_3^r$} &\multicolumn{3}{|c|}{$\bar{\bm{U}}_2^r\equiv \bar{\bm{U}}_1^r$}&\multicolumn{3}{|c|}{$\bar{\bm{U}}_3^b\equiv \bar{\bm{U}}^b_1$}\\
\hline
$\#$&d.&s.f.&$\#$&d.&s.f.&$\#$&d.&s.f.&$\#$&d.&s.f.&$\#$&d.&s.f.&$\#$&d.&s.f.\\
\hline
4&4&\ding{51}&2&3&\ding{51}&2&3&\ding{51}&4&4&\ding{51}&4&4&\ding{51}&2&3&\ding{51}\\
8&5&\ding{51}&4&4&\ding{51}&4&4&\ding{51}&8&5&\ding{51}&8&5&\ding{51}&4&4&\ding{51}\\
16&6&\ding{51}&8&5&\ding{51}&4&7&\ding{51}&8&8&\ding{51}&16&6&\ding{51}&8&5&\ding{51}\\
8&8&\ding{51}&4&7&\ding{51}&&&&&&&8&8&\ding{51}&4&7&\ding{51}\\
32&9&\ding{51}&16&8&\ding{51}&&&&&&&32&9&\ding{51}&16&8&\ding{51}\\
16&12&\ding{55}&8&11&\ding{55}&&&&&&&16&12&\ding{55}&8&11&\ding{55}\\
\hline
\end{tabular}}
\end{center}
\caption{Polynomial structure of the EUs for the original MID, $\bar{\bm{U}}_j$, for the one after the reversal of the arc $(Y_2,Y_3)$, $\bar{\bm{U}}_j^r$ and for the one after the removal of the barren node $Y_2$, $\bar{\bm{U}}_j^b$. The symbol $\#$ corresponds to the number of monomials, d. to the degree and s.f. to whether or not they are square free. }\label{ciao}
\end{table}

The sufficiency principle can be equally stated for a vector of variables~\cite{smith89a,smith89}. However, we can simply apply the criterion in Proposition~\ref{suff} for each variable of the vector and obtain the same result.
\begin{example} 
The MID in Fig.~\ref{fig-ex} is already moralized. Any path from $Y_2$ into $U_i$, $i\in[3]$, goes through both $Y_3$ and $Y_4$. By Proposition~\ref{suff}, we can delete $Y_2$ and the modified diagram is given on the right of Fig.~\ref{fig:1}. This happens to be equal to the diagram resulting from the reversal of the arc $(Y_2,Y_3)$ and the deletion of~$Y_2$.
\end{example}
We now formalize how this principle changes our parametrization.
\begin{proposition}
\label{parsuff}
Let $i,j,k\in\mathbb{V}$ and $G$ be an MID.
Let $Y_i$ be a vertex removed after the application of the sufficiency principle to $G$ and $G'$ the obtained  MID.
Assume $Y_i$ to be the father of $Y_k$ and 
a parent (not the father) of $Y_j$ in $G$ and let $\Pi'_k$ be the parent set of a vertex $Y_k$ in $G'$.  Then the reparametrization of the  MID  with graph  $G'$ is
\begin{align*}
p'_{ky_k\pi'_k}&=\sum_{y_i\in\mathcal{Y}_i}p_{ky_k\pi_k}p_{iy_i\pi_i},\\
p'_{jy_j\pi'_j}&=\sum_{y_j\in\mathcal{Y}_j}p_{jy_j\pi_j}\frac{\prod_{l\in \Pi_{j}\setminus[i-1]}\sum_{\bm{\mathcal{Y}}_{\Pi_l\cap\Pi_k\cap\Pi_i}} p_{ly_l\pi_l}p_{iy_i\pi_i}}{\sum_{y_i\in\mathcal{Y}_i}\prod_{l\in \Pi_j\setminus[i-1]}\sum_{\bm{\mathcal{Y}}_{\Pi_l\cap\Pi_k\cap\Pi_i}} p_{ly_l\pi_l}p_{iy_i\pi_i}}
\end{align*}
\end{proposition}
The proof of this proposition is provided in Appendix~\ref{teorema}.
Again, this new parame-trization $\bm{p}'$  implies a change in the EU vectors.
\begin{lemma}
\label{polsuff}
Assume the vertex $Y_i$ is removed using the sufficiency principle and that $Y_j$ is the child of $Y_i$ with the highest index.  Under the notation of Theorem~\ref{polyexp} the EU vectors in $G'$ are such that
\begin{enumerate}
\item \label{bullet1} for $k<i$, the entries of $\bar{\bm{U}}_k'$ have $r_{kba}/r_i$ monomials of degree $d_{kba}-1$, whilst for $k>i$ their structure does not change.
\item \label{bullet2} for $k\leq j$, $\bar{\bm{U}}_k$ has now dimension $\prod_{s\in C_k}r_s$, where $C_k=B_k\cup \Pi_i\setminus \{k,\dots,n\}$, whilst for $k>j$ its dimension does not change.
\end{enumerate}
\end{lemma}
\begin{proof}
Item 1 of Lemma~\ref{polsuff} is a straightforward consequence of Proposition~\ref{polyop}, since the deletion of the vertex $Y_i$ entails one less EUMarginalization during Algorithm~\ref{algo}. Item 2 of Lemma~\ref{polsuff} follows from the fact that the sets $B_k$ and $C_k$ only affect the dimension of the EU vectors. \end{proof}

Since the application of the sufficiency principle to the diagram of Fig.~\ref{fig-ex} provides the same output network as the one obtained from the reversal of the edge $(Y_2,Y_3)$ and the removal of $Y_2$, an illustration of these results can be found in Table~\ref{ciao}.

\section{Asymmetric decision problems}
\label{asymmetry}
The new symbolic representation of decision problems we introduce here enables us to concisely express a large amount of information that might not be apparent from an ID.  Different types of extra information, often consisting of asymmetries of various kinds, have been explicitly modelled in graphical extensions of the ID model~\cite{Bhattacharjya2012,Bielza2011,Demirer2006,Jensen2006,Smith1993} and are found in the descriptions of many applied decision problems. Although providing a framework for the evaluation of more general decision problems, many of these extensions lose the intuitiveness and the simplicity associated with IDs. Within our symbolic approach we are able to elegantly and concisely characterize asymmetric decision problems through manipulations of the polynomials representing the ID's EU as we show next.

Asymmetries can be categorized in three classes. If the possible outcomes or decision options of a variable vary depending on the past, the asymmetry is called \textit{functional}. If the very occurrence of a variable depends on the past, the asymmetry is said to be \textit{structural}. \textit{Order} asymmetries are present if $\{Y_i:i\in\mathbb{D}\}$ is not totally ordered. In this section we only deal with functional asymmetries.  Heuristically, for a functional asymmetry the observation of $\bm{y}_A$, $A\subset [n]$, restricts the space $\bm{\mathcal{Y}}_B$ associated to a vector $\bm{Y}_B$, such that $A\cap B=\emptyset$. This new space, $\bm{\mathcal{Y}}_B'$ say, is a subspace of $\bm{\mathcal{Y}}_B$. 

In Theorem~\ref{polyasy} we characterize an asymmetry between two chance nodes and, depending on the stage of the evaluation,  this may  entail setting equal to zero monomials in either some or all  the rows of the EU vector.  We present the result for elementary asymmetries of the following form: if $Y_i=y_i$ then $Y_j\neq y_j$. Composite asymmetries are unions of simple asymmetries and the features of the EU vectors in more general cases can be deduced through a sequential application of Theorem~\ref{polyasy}.

\begin{theorem}
\label{polyasy}
Let $G$ be an MID, $Y_i$ and $Y_j$ be two random variables with $j>i$, $U_x$ be the utility node following $Y_j$ in the DS. Assume the asymmetry $Y_i=y_i\Rightarrow Y_j\neq y_j$ holds and that $k$ and $z$ are the highest indices such that $j\in B_k$ and $i\in B_z$ and assume  $k>j$. Then 
\begin{itemize}
\item for $j<t\leq z$,  $\bar{\bm{U}}_t$ has $\prod_{s\in B_t\setminus \{i\cup j\}} r_s$ rows with no monomials;
\item for $i<t\leq j$, $\bar{\bm{U}}_t$ has $\prod_{s\in B_t\setminus \{i\}} r_s$ rows with polynomials all with a different structure. Specifically, these consists, in the notation of Theorem~\ref{polyexp}, of $s_{tba}$ monomials of degree $d_{tba}$, where, for $a=x,\dots,m$ and $b=l,\dots,a$, 
\[
s_{tba}=\left(\binom{a-x}{b-l}-1\right)\prod_{s=t}^{j_a}r_s/r_j;
\]  
\item for $t\leq i$, each row of $\bar{\bm{U}}_t$ has in the notation of Theorem~\ref{polyexp}, $f_{tba}$ monomials of degree $d_{tba}$, where for $a=x,\dots,m$ and $b=l,\dots,a$ 
\[
f_{tba}=\left(\binom{a-x}{b-l}-1\right)\prod_{s=t}^{j_a}r_s/(r_j\cdot r_i).
\]  
\end{itemize}

\end{theorem}
The proof of this theorem is provided in Appendix~\ref{cippu}. Corollary~\ref{corro} gives a  characterization of simple asymmetries between any two variables, whether they are controlled or non-controlled. This follows  from Theorem~\ref{polyasy} since controlled variables can be thought of as a special case of random ones.

\begin{corollary}\label{corro}
In the notation of Theorem~\ref{polyexp} and under the assumptions of Theorem~\ref{polyasy}, with the difference that $Y_i$ and $Y_j$ are two variables, controlled or non-controlled, we have that
\begin{itemize}
\item for $j<t\leq z$, each row of $\bar{\bm{U}}_t$ has $\prod_{s\in B_t\setminus \{i\cup j\}} r_s$ rows with no monomials;
\item for $i<t\leq j$, $\bar{\bm{U}}_t$ has at most $\prod_{s\in B_t\setminus \{i\}} r_s$ rows with polynomials  all with a different structure. Specifically, these consists of between $s_{tba}$ and $r_{tba}$ monomials of degree $d_{tba}$, for $a=x,\dots,m$ and $b=l,\dots,a$;
\item  for $t\leq i$, some rows of $\bar{\bm{U}}_t$ have a number of monomials of degree $d_{tba}$ between $f_{tba}$ and $r_{tba}$, for $a=x,\dots,m$ and $b=l,\dots,a$.
\end{itemize} 
\label{corasy}
\end{corollary}

\begin{example}[Example~\ref{nuclearexample} continued] 
\label{eee}
Assume that the DM believes the decision problem is characterized by three composite asymmetries:
\begin{itemize}
\item if $Y_1$ was fixed to $1$, then $Y_4=1$ cannot be chosen;
\item if either $Y_2$ or $Y_3$ were observed to be equal to $1$ then $Y_5=1$;
\item if  $Y_4=1$ then both $Y_5$ and $Y_6$ are equal to $1$. 
\end{itemize}
\begin{figure} 
\begin{center}
\scalebox{0.85}{
\begin{tikzpicture}[node distance=4cm]
\node[draw,rectangle] (00) at (0,0) {$Y_1$};
\node[draw,circle] (1-1) at (2,-1) {$Y_3$};
\node[draw,circle] (11) at (2,1) {$Y_2$};
\node[draw,rectangle] (20) at (4,0) {$Y_4$};
\node[draw,circle] (31) at (8,1) {$Y_6$};
\node[draw,circle] (3-1) at (8,-1) {$Y_5$};
\node (0-1) at (0,-1) {$U_1$};
\node (21) at (4,1) {$U_3$};
\node (4-1) at (10,-1) {$U_2$};
\draw[->] (00) -- (11);
\draw[->] (00) -- (1-1);
\draw[->,anchor=center, below] ([yshift=0.1cm]00.east) to node {} ([yshift=0.1cm]20.west);
\draw[->,dashed,anchor=center, below] ([yshift=-0.1cm]00.east) to node {\tiny{$Y_1=4|Y_4=0$}} ([yshift=-0.1cm]20.west);
\draw[->] (11) -- (1-1);
\draw[->] (1-1) -- (0-1);
\draw[->] (20) -- (21);
\draw[->] (31) -- (21);
\draw[->] (3-1) -- (4-1);
\draw[->] (11) -- (20);
\draw[->] (1-1) -- (20);
\draw[->] (1-1) -- (3-1);
\draw[->,anchor=center, below] (20.north) to node {} (31);
\draw[->,anchor=center, below] (20.south) to node {} (3-1);
\draw[->] (3-1) -- (31);
\node[draw=black, dashed, ellipse, inner sep =4pt, fit= (11) (1-1)] (circ1) {};
\node[draw=black, dashed, ellipse, inner sep =4pt, fit= (31) (3-1)] (circ2) {};
\draw[->, dashed,anchor=center, below] ([yshift=-1.2cm]circ1.east) to node {\tiny{$Y_2=1,Y_3=1|Y_5=1$}} ([yshift=-0.2cm]3-1.west);
\draw[->, dashed,anchor=center, above] (20) to node {\tiny{$Y_4=1|Y_5=1,Y_6=1$}} (circ2.west);
\end{tikzpicture}
}
\end{center}
\caption{Representation of the asymmetric version of the MID of Fig.~\ref{fig-ex} through a sequential influence diagram. }\label{sid}
\end{figure}
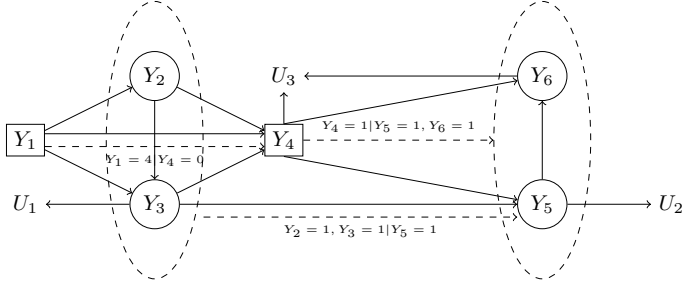
A graphical representation of these asymmetries is given in Fig.~\ref{sid}, in the form of a sequential influence diagram~\cite{Jensen2006}. Asymmetries are represented as labels on new dashed arcs. If the asymmetry is composite, then vertices can be grouped through a dashed ellipse and dashed arcs can either start or finish by the side of these ellipses. Although this generalization of the MID in Fig.~\ref{fig-ex} graphically captures the asymmetries, most of its transparency is now lost. 
Instead asymmetries have the opposite effect on our polynomial representation of MIDs by greatly simplifying the  structure of the EUs.
 
In this asymmetric framework the first row of $\bar{\bm{U}}_6$ corresponds to $k_3\psi_{311}p_{6111}$, whilst its second row is empty.  
This is because according to Theorem~\ref{polyasy} the monomial $k_3\psi_{301}p_{6011}$ in equation~(\ref{eq:ciao}) is cancelled by the asymmetry $Y_4=1\Rightarrow Y_6=1$, $k_3\psi_{311}p_{6101}$ by $Y_4=1\Rightarrow Y_5=1$ and $k_3\psi_{301}p_{6001}$ by both asymmetries. The imposition of asymmetries  further reduces from ten to three the number of monomials in $\bar{\bm{U}}_5$ which becomes 
\[
k_3\psi_{311}p_{6111}p_{511i}+k_2\psi_{21}p_{511i}+hk_2k_3\psi_{311}\psi_{21}p_{6111}p_{511i}, \;\;\;\;\; i=0,1.
\]
Suppose the DM decided to fix $Y_4=0$ if $Y_3=1$ and $Y_4=1$ if $Y_3=0$.
The entry of $\bar{\bm{U}}_3$ for which $Y_2=1$ and $Y_1=1$ can be written as the sum of the terms 
\begin{align*}
&(k_2\psi_{21}+k_3\psi_{311}p_{6111}(1+kk_2\psi_{21}))p_{5110}p_{3011}+k_1(\psi_{10}p_{3011}+\psi_{11}p_{3111})\\
&kk_1k_3\psi_{11}p_{5101}p_{3111}((1+k_2\psi_{21})(\psi_{300}p_{6010}+\psi_{310}p_{6110})).
\end{align*}
This polynomial consists of only nine monomials. This compared with the number of monomials in the symmetric case, $42$ (see Table~\ref{ciao}), means that even in this small problem  the number of monomials is decreased by over three quarters.
\end{example}

So the example above illustrates that under asymmetries the polynomial representation is simpler than standard methods but still able to inform decision centres about the necessary parameters to elicit. A more extensive discussion of the advantages of symbolic approaches in asymmetric contexts, although fully inferential ones, can be found in~\cite{Gorgen2015}.  Finally it is possible to develop a variant of Algorithm~\ref{algo1} which explicitly takes into account the asymmetries of the problem \textit{during} the computation of the EU vectors.  Note that this approach  would  be  computationally  more efficient, since this would require the computation of a smaller number of monomials/polynomials.

\section{An example} \label{sect:example}
\label{sec:example} 
\begin{table}
\begin{center}
\resizebox{\columnwidth}{!}{
\begin{tabular}{|llll||llll|}
\hline
\multicolumn{8}{|c|}{\textbf{Parameters' specifications}}\\
\hline
\multicolumn{4}{|c||}{\textit{Complete}}&\multicolumn{4}{|c|}{\textit{Partial}}\\
\hline
$p_{6111}=0.3$,&$p_{5110}=0.2$,&$\psi_{311}=0$,&$k_3=0.4$&$p_{6100}=0.3$&$\psi_{311}=0,$&$\psi_{21}=0,$&$k_3=0.4$,\\
$p_{6110}=0.2$,&$p_{5101}=0.9$,&$\psi_{310}=0.4$,&$k_2=0.2$&$p_{5110}=0.2$,&$\psi_{310}=0.4$&$\psi_{20}=1,$&$k_2=0.2,$\\
$p_{6101}=0.2$,&$p_{5100}=0.6$,&$\psi_{301}=0.8$,&$k_1=0.2$&$p_{5101}=0.9$,&$\psi_{300}=1,$&&$k_1=0.2,$\\
$p_{6100}=0.3$,&$p_{5111}=0.7$,&$\psi_{300}=1$,&$\,\,h=0.9$&$p_{5100}=0.6$, & & &$\,\,h=0.9$ \\
\multicolumn{2}{|c}{$\psi_{21}=0$,} &\multicolumn{2}{c||}{$\psi_{20}=1$,}&\multicolumn{2}{|c}{$p_{5111}=1-p_{6111}$,}&\multicolumn{2}{c|}{$p_{6001}=p_{6010}$,}\\
\hline
\end{tabular}}
\end{center}
\caption{Complete and partial specification of the parameters associated to MID in Fig.~\ref{fig-ex} for the optimization step over $Y_4$.  By the sum-to-one condition $p_{6001}=p_{6010}$ is equivalent to $p_{6101}=p_{6110}$. 
\label{table:num}}
\end{table}

\begin{figure}
\begin{center}
\scalebox{0.9}{
\includegraphics[scale=0.5]{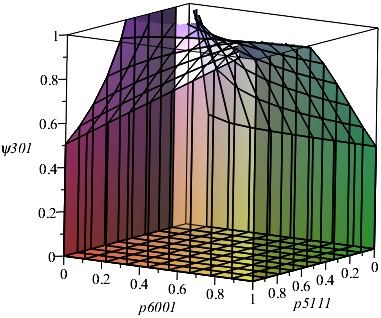}\hspace{1cm}
\includegraphics[scale=0.5]{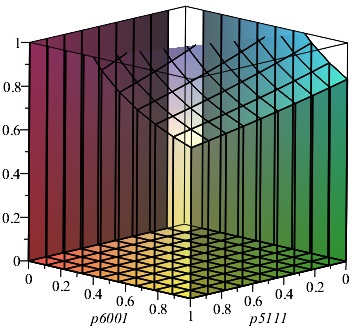}
}
\end{center}
\caption{Admissible domains, expressed in the unknowns $\psi_{301}$, $p_{6001}$ and $p_{5111}$,  for the combinations of parameters leading to a preferred decision $Y_4=0$ (coloured regions) and $Y_4=1$ (white regions) for the MID in Fig.~\ref{fig-ex} given the partial numeric specification in Table~\ref{table:num}.} 
 \label{fig:opt}
\end{figure}

In this section we study the polynomial features of the EUs associated to the MID in Fig~\ref{fig-ex}.  We focus  on the selection of the decision variable $Y_4$  and consider two different scenarios in which the DM provides two different  sets of information of the relevant parameters.   In the first scenario the elicitation is complete, i.e. for each parameter   the DM  delivers  the unique numerical value  specified in the left hand side of Table~\ref{table:num}. The second  scenario combines unique probability specifications, symbolic parameters and qualitative information.  Specifically the DM does not elicit $p_{5111}$, $p_{6001}$, $p_{6010}$, $p_{6011}$ and $\psi_{301}$,  because e.g. there is strong uncertainty related to their values, specifies the two relationships $p_{5111}=p_{6011}$ and $p_{6001}=p_{6010}$  and assigns specific values to the remaining parameters as indicated in Table~\ref{table:num}.

 In the first scenario, using any standard propagation algorithm or by simply substituting the appropriate numerical values from Table~\ref{table:num} into the EU polynomial $\bar{U}_5(y_4,y_3)$ from Table~\ref{table}, the DM would be suggested to choose $Y_4=1$ if $Y_3=0$ and $Y_4=0$ if $Y_3=1$, since
\[ 
\bar{U}_5(1,0)=0.4465,\, \bar{U}_5(0,0)=0.4460 \text{ and } \bar{U}_5(1,1)=0.3074, \, \bar{U}_5(0,1)=0.3755.
\]
In an automated decision making process the DM might overlook the small difference in EU values when $Y_3=0$, which already suggests that small changes in parameters' values may lead to different preferred policies. 

A symbolic study of EUs in the partial elicitation case of the second scenario can provide insights on why the DM's decision making may not be robust. Substituting the partial numeric specification in Table~\ref{table:num} into the polynomial from Table~\ref{table}
yields
\begin{align*}
\bar{U}_5(y_4,y_3) 
= &  0.2 p_{50y_4y_3} + 0.4 \left( \psi_{31y_4} p_{611y_4} + \psi_{30y_4} p_{601y_4} \right) p_{51y_4y_3}  \\
+ & 0.472 \left( \psi_{31y_4} p_{610y_4} + \psi_{30y_4} p_{600y_4} \right) p_{50y_4y_3}
\end{align*}
which is further specialised in 
\begin{align*}
\bar{U}_5(1,1) & =  0.2 (1-p_{5111}) +0.4 \psi_{301} p^2_{5111} + 0.472 \psi_{301} p_{6001} (1-p_{5111}) \\
\bar{U}_5(0,1) & =   0.100976 + 0.6192 p_{6001}\\
\bar{U}_5(1,0) & =0.16 + 0.08 \psi_{301} p_{5111} + 0.3776 p_{6001}  \\
\bar{U}_5(0,0) & = 0.233984 + 0.6192 p_{6001}
\end{align*}
Under the partial specification scenario, the admissible domains when $Y_3=1$,   namely  $\argmax \{ \bar{U}_5(1,1),\bar{U}_5(0,1) \}$,  are reported on the left hand side of Fig.~\ref{fig:opt} and the associated \textit{indifference surface} is defined by the equation $\bar{U}_5(1,1)=\bar{U}_5(0,1)$. 
The right hand side of Fig.~\ref{fig:opt} shows the admissible domains when $Y_3=0$. 
For $Y_3=1$, the combination of values elicited by the DM is well inside the colored region, whilst for $Y_3=0$ it is very closed to the {indifference surface} defined by the points where the DM is indifferent between the two policies. The indifference surface for $Y_3=0$ is very smooth and regular since the associated variety is defined by a simple multilinear polynomial. Conversely, the surface for $Y_3=1$ exhibits more interesting features since the associated variety is defined by a quadratic function.

\begin{figure}
\begin{center}
\includegraphics[scale=0.25]{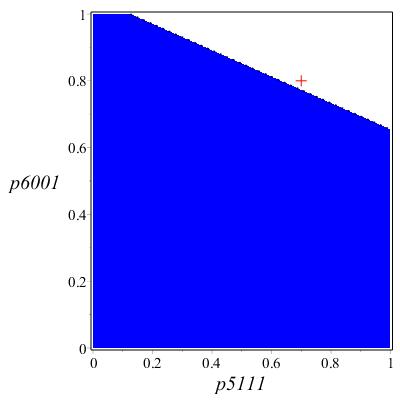}\hspace{0.5cm}
\includegraphics[scale=0.25]{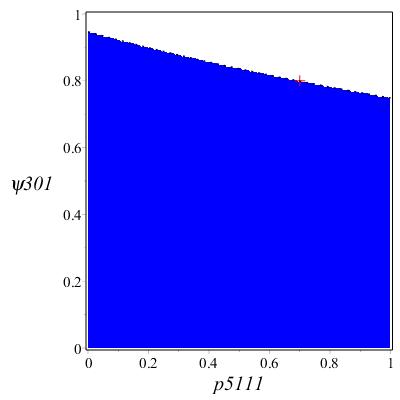}\hspace{0.5cm}
\includegraphics[scale=0.25]{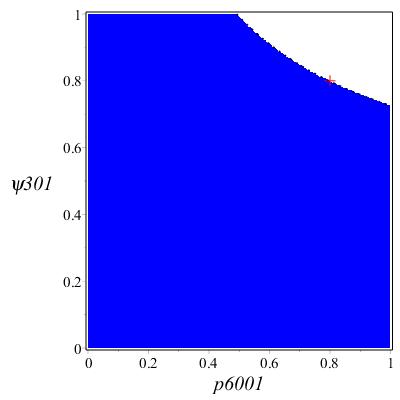}
\end{center}
\caption{Admissible domains for subsets of 2 elements of the parameter space for $Y_3=0$, fixing the third to the value of the complete elicitation (colored for $Y_4=0$ and white for $Y_4=1$). \label{fig:o}}
\end{figure}
\begin{figure}
\begin{center}
\includegraphics[scale=0.25]{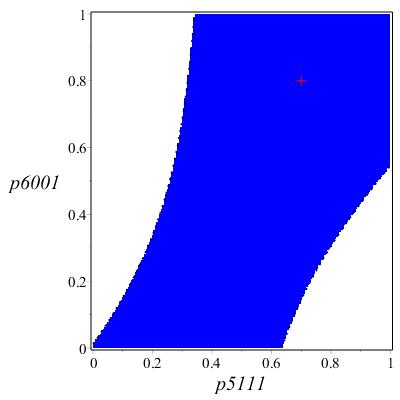}\hspace{0.5cm}
\includegraphics[scale=0.25]{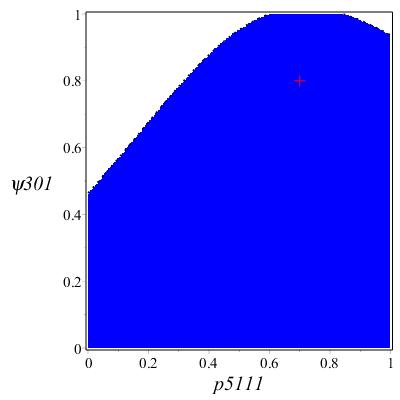}\hspace{0.5cm}
\includegraphics[scale=0.25]{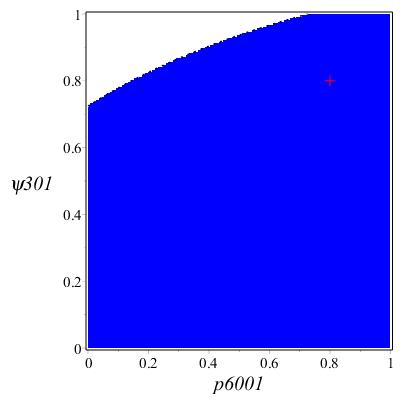}
\end{center}
\caption{Admissible domains for subsets of 2 elements of the parameter space for $Y_3=1$, fixing the third to the value of the complete elicitation (colored for $Y_4=0$ and white for $Y_4=1$). \label{fig:op}}
\end{figure}

Additional information about the DM's decision problem can be gained by investigating the admissible domains defined by two parameters only, when the third one is fixed to the value chosen in the complete elicitation scenario. In Fig.~\ref{fig:o} we report the regions for $Y_3=0$. The admissible domains are very \lq\lq{s}mooth\rq\rq{} and the indifference surfaces are all monotonic functions. In all the plots the complete elicitation point is very close to the indifference curve {and thus small perturbations of the parameters can lead to a different preferred policy}.  Much more robust is the DM's potential decision in the case $Y_3=1$, since all the complete elicitation points are well inside the admissible domains. Note how in this case the admissible domains have a much more complex geometry, due to the polynomial structure of the indifference surface variety. We highlight a few points from this example:
\begin{itemize}
\item  the geometry of the admissible regions has provided insights on the decision making process. In much more complex problems,  tools of algebraic geometry~\cite{Cox2007a}  can still be used to guide DMs and to uncover even more surprising features;
\item although other approaches allow for non exact probability specifications, our symbolic characterization provides a straightforward platform to input qualitative information, e.g. equality of two parameters, which entails a simple reparametrization of the problem;
\item the symbolic approach is particularly efficient for this type of sensitivity studies since a DM can simply plug-in different combinations of values for the unknowns and instantaneously observe the results. In full numerical domains, the propagation of EUs would need to performed for each combination of values and this can become computationally very expensive;
\item if, after robustness studies as the one in this example, the DM is still not convinced about  a preferred course of action,  our  algorithm can be adapted to run separately for each admissible domain back to the root of the MID. In this way it would then output the admissible regions for each multivariate available policy together with its defining polynomial;
\item the identification of the admissible domains consists of the solution of a system of polynomial inequalities. We are currently investigating these domains using semi-algebraic methods~\cite{parrilo2000};
\item all these methods are especially informative in asymmetric domains, since different policies can be associated to polynomial having very different properties. This is because, as shown in Theorem~\ref{polyasy}, in contast to standard MIDs, different policies can be associated with very differently structured polynomials.
\end{itemize}

\section{Discussion}
\label{discussion}
With this work we have developed symbolic methods - currently being successfully applied to the analysis of probabilistic graphical models - to study  MIDs. We have defined a complete toolkit to deal with standard operations for MIDs from a symbolic point of view, such as the computations of EUs, possible manipulations of the diagram and asymmetries. Whilst in open-loop analyses our symbolic definition finds its natural application, in closed-loop  analyses the EU-Maximization operation becomes critical. In some specific cases, as illustrated through the example in Section \ref{sec:example}, partial parameters' elicitations will allow the DM to perform such step. In more general cases we still need to formalize such maximization techniques  for example by adopting semi-algebraic methods  which have already proved successful in other applications~\cite{parrilo2000}. We expect these to be particularly useful in asymmetric domains, since  different polynomial structures can inform even more deeply the DM about  the structure of the decision space.

We here provide a full report of  an implementation of our methodology within an accessible computer algebra system. Of course when addressing very large problems  generic tools can have difficulties handling the number of unknown variables that need to be stored in the computer memory and computations  may become infeasible. However,  there are ways around this memory problem. For example by imposing certain conditions on the model -  formally discussed in~\cite{Leonelli2015} - computations can then be distributed. This can dramatically reduce complexity and make calculations  again feasible albeit with the necessary addition of further software - designed for the particular application - which intelligently merges together the outputs of the different contributing distributed components of the system: see also~\cite{Smith2015}. The simulations carried out in Section~\ref{sect:simulation} and the example in Section~\ref{sect:example} showed that the methodology and algorithm presented in this paper are competitive and allow for the analysis of more general classes of models than traditional methods. Implementations based on specialised programs rather than on a general purpose software like Maple\textsuperscript{\textit{\tiny{TH}}} will enable the analysis of more complex MIDs.

 Also in the case the  variables take values in continuous spaces,
EU  exhibits a similar polynomial representation to the one discussed in this paper for discrete variables. In  the continuous case the unknown quantities of the polynomials are low order moments. Examples of these polynomials are presented in~\cite{Leonelli2015}. Just as in  the discrete case,  the manipulations of the diagrams for policies with continuous variables and their associated asymmetries can be described as operations over the polynomials. A full study of the symbolic representation of EUs in a continuous domain will be reported in future work. 



\appendix
\section{Proofs}
\subsection{Proof of Proposition~\ref{i-th}}
We develop the proof via backward induction over the random and decision vertices of the MID, starting from $Y_n$. Define, for $i\in[n]$,
\begin{equation*}
\hat{U}_i=\int_{\mathcal{Y}_{[n]_{i-1}^{\mathbb{V}}}}\max_{\mathcal{Y}_{[n]_{i-1}^{\mathbb{D}}}}\sum_{I\in\mathcal{P}_0([m])}h^{n_I-1}\prod_{i\in I}k_iU_i(\bm{y}_{P_i})f\left(\bm{y}_{[n]_{i-1}^{\mathbb{V}}}\;|\;\bm{y}_{[i{-1}]}\right)\dr\bm{y}_{[n]_{i-1}^{\mathbb{V}}},
\end{equation*}
where $[n]_{i-1}^{\mathbb{V}}=[n]\setminus[i-1]\cap \mathbb{V}$, $[n]_{i-1}^{\mathbb{D}}=[n]\setminus[i-1]\cap \mathbb{D}$ and $\Pi_{[n]_{i-1}^{\mathbb{V}}}=\cup_{j\in[n]_{i-1}^{\mathbb{V}}}\Pi_j$. The quantity $\hat{U}_i$ corresponds to an overall EU score after having marginalized/maximized $Y_i,\dots,Y_n$.

The DM's preferences are a function of $Y_n$ only through $k_mU_m(\bm{y}_{P_m})$, since by construction $n=j_m\in\mathbb{J}$. Therefore this quantity can be either maximized or marginalized as in equation~(\ref{cond4}) to compute $\bar{U}_n(\bm{y}_{B_n})$. Note that $B_n$ includes only the indices of the variables $\bar{U}_n$ formally depends on, since $B_n=P_m\setminus \{n\}$, if $n\in\mathbb{D}$, whilst $B_n=P_m\cup \Pi_n\setminus \{n\}$, if $n\in\mathbb{V}$. Then 
\begin{equation*}
\label{eq:proofpro}
\hat{U}_n=\sum_{I\in\mathcal{P}_0([m])}h^{n_I-1}\prod_{i\in I}\left(\mathbbm{1}_{\{i\neq n\}}\left[ k_iU_i(\bm{y}_{P_i})\right]+\mathbbm{1}_{\{i=n\}}\left[\bar{U}_i(\bm{y}_{B_i})\right]\right).
\end{equation*}
Now consider $Y_{n-1}$. If $n-1\not\in \mathbb{J}$, then $\hat{U}_n$ is a function of $Y_{n-1}$ only through $\bar{U}_n$. Therefore maximization and marginalization steps can be computed as in equation~(\ref{cond1}) to compute $\bar{U}_{n-1}(\bm{y}_{B_{n-1}})$. Again $B_{n-1}$ includes the indices of the variables $\bar{U}_{n-1}$ formally depends on, since $B_{n-1}=P_m\setminus\{n,n{-1}\}$, if $n,n{-1}\in\mathbb{D}$, $B_{n-1}=P_m\cup\Pi_n\cup\Pi_{n-1}\setminus\{n,n{-1}\}$, if $n,n{-1}\in\mathbb{V}$, $B_{n-1}=P_m\cup\Pi_{n-1}\setminus\{n,n{-1}\}$, if $n\in\mathbb{D}$ and $n{-1}\in\mathbb{V}$,  $B_{n-1}=P_m\cup\Pi_n\setminus\{n,n{-1}\}$, if $n\in\mathbb{V}$ and $n{-1}\in\mathbb{D}$. Then 
\begin{equation*}
\hat{U}_{n-1}=\sum_{I\in\mathcal{P}_0([m])}h^{n_I-1}\prod_{i\in I}(\mathbbm{1}_{\{i\neq n\}} k_iU_i(\bm{y}_{P_i})+\mathbbm{1}_{\{i=n\}}\bar{U}_{i-1}(\bm{y}_{B_{i-1}})).
\end{equation*}
 Conversely, if $n-1\in\mathbb{J}$, $\hat{U}_n$ is potentially a function of $Y_{n-1}$ through both $U_{m-1}(\bm{y}_{P_{m-1}})$ and $\bar{U}_n(\bm{y}_{B_n})$ and note that $\hat{U}_n$ can be written in this case as
\begin{equation*}
\hat{U}_n=\sum_{I\in\mathcal{P}_0([m{-2}])}h^{n_I-1}\prod_{i\in I}k_iU_i(\bm{y}_{P_i})+U_{m-1}'+\left(\sum_{i\in\mathcal{P}_0([m{-2}])}h^{n_i-1}\prod_{i\in I}k_iU_i(\bm{y}_{P_i})\right)U_{m-1}', 
\end{equation*}  
where 
\begin{equation*}
U_{m-1}'=hk_{m-1}U_{m-1}(\bm{y}_{P_{m-1}})\bar{U}_n(\bm{y}_{B_n})+k_{m-1}U_{m-1}(\bm{y}_{P_{m-1}})+\bar{U}_n(\bm{y}_{B_n}).
\end{equation*}
Therefore optimization and marginalization steps can be performed over $U_{m-1}'$ as specified in the two equations~(\ref{cond23}) respectively. Then note that $\hat{U}_{n-1}$ can be written as 
\begin{align*}
\hat{U}_{n{-1}}&=\hspace{-0.4cm}\sum_{I\in\mathcal{P}_0([m{-2}])}\hspace{-0.4cm}h^{n_I-1}\prod_{i\in I}k_iU_i(\bm{y}_{P_i})+\bar{U}_{n-1}(\cdot)+\left(\sum_{i\in\mathcal{P}_0([m{-2}])}\hspace{-0.4cm}h^{n_i-1}\prod_{i\in I}k_iU_i(\bm{y}_{P_i})\right)\bar{U}_{n-1}(\cdot)\\
&=\hspace{-0.4cm}\sum_{I\in\mathcal{P}_0([m{-1}])}\hspace{-0.4cm}h^{n_I-1}\prod_{i\in I}(\mathbbm{1}_{\{i\neq n-1\}} k_iU_i(\bm{y}_{P_i})+\mathbbm{1}_{\{i=n-1\}}\bar{U}_i(\bm{y}_{B_i})).
\end{align*}

Now for a $j\in[n{-2}]$ and assuming with no loss of generality that $k$ is the index of a utility vertex such that $j_{k-1}<j\leq j_k$, we have that 
\begin{equation*}
\label{eq:proofproof}
\hat{U}_j=\sum_{I\in\mathcal{P}_0([k])}h^{n_I-1}\prod_{i\in I}(\mathbbm{1}_{\{i\neq j\}} k_iU_i(\bm{y}_{P_i})+\mathbbm{1}_{\{i=j\}}\bar{U}_i(\bm{y}_{B_i})).
\end{equation*}
Therefore at the following step, when considering $Y_{j-1}$, we can proceed as done with $Y_{n-1}$ by maximization and marginalization in equations~(\ref{cond23})-(\ref{cond1}) to compute $\hat{U}_{j-1}$. Thus at the conclusion of the procedure, $\bar{U}_1$ yields the EU of the optimal decision.
\label{proof:proof}
\subsection{Proof of Theorem~\ref{polyexp}}
\label{appendix1}
For a subset $I\in \mathcal{P}_0([m])$, let $j_I$ be the index of the variable appearing before the utility vertex with index $U_{\max_{I}}$ in the decision sequence. Let $C_{i,I}=\{z\in\mathbb{V}:i\leq z\leq j_I\}$ and recall that $l$ is the index of the first utility node following $Y_i$ in the DS. The EU function of equation~(\ref{cond4})-(\ref{cond1}) can be (less intuitively) written as $\bar{U}_i(\bm{y}_{B_i})=\sum_{I\in \mathcal{P}_0(\{l,\dots,m\})}\bar{U}_{i,I}(\bm{y}_{B_i})$, where  $\bar{U}_{i,I}(\bm{y}_{B_i})$ is defined as
\begin{equation}
\label{proof1}
\bar{U}_{i,I}(\bm{y}_{B_i})=\sum_{I\in \mathcal{P}_0(\{l,\dots,m\})}h^{n_I-1}\prod_{s\in I}k_sU_s(\bm{y}_{P_s})\sum_{\bm{y}_{C_{i,I}}\in\bm{\mathcal{Y}}_{C_{i,I}}}\prod_{j\in C_{i,I}}P(y_j|\bm{y}_{\Pi_j}).
\end{equation}
The EU therefore depends on the power set of the indices of the utility vertices subsequent to $Y_i$ in the decision sequence. We can note that for any $I,J\in \mathcal{P}(\{l,\dots,m\})$ such that $\# I=\# J$ and $U_{\max_{I}}=U_{\max_{J}}$, $\bar{U}_{i,I}(\bm{y}_{B_i})$ and $\bar{U}_{i,J}(\bm{y}_{B_i})$ have the same polynomial structure since $C_{i,I}=C_{i,J}$. Now for $a=l,\dots,m$ and $b=l,\dots,a$, by the properties of binomial coefficients, $\binom{a-l}{b-l}$ counts the number of elements $I\in\mathcal{P}_0(\{l,\dots,m\})$ having $\# I=b-l+1$ and including $a$. Thus $r_{iba}$ in equation~(\ref{struct}) counts the correct number of monomials having a certain degree since $\bm{\mathcal{Y}}_{C_{i,I}}=\times_{t\in C_{i,I}}\mathcal{Y}_t$. Further note that considering each combination of $b$ and $a$ in the ranges specified above, we count each element of $\mathcal{P}_0(\{l,\dots,m\})$.

 By having a closer look at $d_{iba}$ in equation~(\ref{struct}) it is easy to deduce the corresponding degree of these monomials. The first term of $d_{iba}$, $(b-l)$, computes the degree associated to the criterion weight $h$, since $b-l=n_I-1$ and the second term, $2(b-l+1)$, computes the degree associated to the product between the criterion weights $k_s$ and the utilities $U_s(\bm{y}_{P_s})$ for $s\in C_{i,I}$. The last term $w_{ia}$ corresponds to the degree deriving from the probabilistic part of equation~(\ref{proof1}), which is equal to the number of non-controlled vertices between $Y_i$  and $Y_{j_{\max_{I}}}$ (both included).

Since the set $B_i$ includes the arguments of $\bar{U}_i(\bm{y}_{B_i})$ and $\bm{\mathcal{Y}}=\times_{i\in[n]}\mathcal{Y}_i$, equation~(\ref{veccond}) guarantees that the dimension of the EU vector is $\prod_{t\in B_i}r_t$.

\subsection{Proof of Proposition~\ref{par}.}
After the reversal of the arc $(Y_i,Y_j)$ into $(Y_j,Y_i)$, the new parent sets of these two variables are $\Pi'_j=\{\Pi_j \cup \Pi_i\setminus i\}$ and $\Pi'_i=\{j\cup \Pi_i\cup \Pi_j\setminus i\}$. Call $\Pi_{j\setminus i}=\{\Pi_j\setminus i\}$.  It then follows that
\begin{align*}
p_{iy_i\pi'_i}&=P(y_i\;|\; \bm{y}_{\Pi'_i})=P(y_i\;|\; \bm{y}_{\Pi_{j\setminus i}}, \bm{y}_{\Pi_i},y_j)=\frac{P(y_j\;|\; \bm{y}_{\Pi_{j\setminus i}}, \bm{y}_{\Pi_i},y_i)P(y_i\;|\; \bm{y}_{\Pi_{j\setminus i}}, \bm{y}_{\Pi_i})}{P(y_j\;|\; \bm{y}_{\Pi_{j\setminus i}}, \bm{y}_{\Pi_i})}\\
&=\frac{P(y_j\;|\;\bm{y}_{\Pi_j})P(y_i\;|\;\bm{y}_{\Pi_i})}{P(y_j\;|\;\bm{y}_{\Pi_{j\setminus i}},\bm{y}_{\Pi_i})}\hspace{-0.05cm}=\hspace{-0.05cm}\frac{P(y_j\;|\;\bm{y}_{\Pi_j})P(y_i\;|\;\bm{y}_{\Pi_i})}{\sum_{y_i\in\mathcal{Y}_i}\hspace{-0.075cm}P(y_j\;|\;y_i,\bm{y}_{\Pi_{j\setminus i}})P(y_i\;|\;\bm{y}_{\Pi_i})}\hspace{-0.05cm}=\hspace{-0.05cm}\frac{p_{jy_j\pi_j}p_{iy_i\pi_i}}{\sum_{y_i\in\mathcal{Y}_i}\hspace{-0.05cm}p_{jy_j\pi_j}p_{iy_i\pi_i}},
\end{align*}
and
\begin{equation*}
p_{jy_j\pi'_j}'=P(y_j\;|\; \bm{y}_{\Pi'_j})=P(y_j\;|\; \bm{y}_{\Pi_{j\setminus i}},\bm{y}_{\Pi_i})=\hspace{-0.1cm}\sum_{y_i\in\mathcal{Y}_i}P(y_j\;|\; \bm{y}_{\Pi_j})P(y_i\;|\;\bm{y}_{\Pi_i})=\hspace{-0.1cm}\sum_{y_i\in\mathcal{Y}_i}p_{jy_j\pi_j}p_{iy_i\pi_i}.
\end{equation*}

The proof of the barren node removal easily follows from the fact that the vertex is not included anymore in the MID.
\label{proofteo2}
\subsection{Proof of Lemma~\ref{polarc} and \ref{br}.}
\label{corollario}
We first consider the arc reversal and the change of dimension of the vectors. If $j\not\in \mathbb{J}$ the sets $B_k$ that are affected by the arc reversal are only the ones such that  $k\in\Pi_i\cup \Pi_j$ and the set $B'_k$ simply takes into account the presence of the additional edges in $G'$. If $j\in \mathbb{J}'$ then the sets $B_k$ affected by the arc reversal are the ones such that $k\in \Pi_i\cup \Pi_j\cup P_{j_j}$ and the set $B''_k$ additionally takes into account that the indices in $P_{j_j}$ are included only before the EUMarginalization between $\bar{\bm{U}}_{i+1}$ and $\bm{p}_j$. The final case is if $j\not \in \mathbb{J}'$, which can be seen as a combination of the previous two cases.

Now consider the polynomial structure of the entries after an arc reversal. If $j\not\in \mathbb{J}$, then the adjusted Algorithm~\ref{algo} simply computes an EUMarginalization between $\bar{\bm{U}}_{j+1}$ and $\bm{p}_i$ instead of $\bm{p}_j$. Therefore the entries of $\bar{\bm{U}}_j$ have $r_{jba}'=r_ir_{(j+1)ba}/r_j$ monomials of degree $d_{(j+1)ba}$ and, until the adjusted algorithm computes $\bar{\bm{U}}_i$, the change in the structure is propagated through the \lq{E}UOperations\rq{.} If $j\in\mathbb{J}'\cap\mathbb{J}$, then instead of an EUMultiSum and a EUMarginalization, now the algorithm only computes an EU-Marginalization and, as before, the change is propagated until $\bar{\bm{U}}_i$. As in the previous paragraph, the last case can be seen as combination of the previous two situations.

Consider now the deletion of the barren node $Y_i$. The set $B_z$ is the one with the highest index which includes $i$ in $G$. Thus, for $i<k\leq z$, $i\in B_k$ and $\bar{\bm{U}}_k$ is conditional on $Y_i=y_i$. The deletion of this vertex therefore implies that the dimension of the vector becomes $c_k'/r_i$. For $k\leq i$, Algorithm~\ref{algo} now performs one EUMarginalization less and, from Proposition~\ref{polyop}, we deduce that $\bar{U}_k'$ has now $r_{kba}/r_i$ monomials of degree $d_{kba}-1$.
 
 \subsection{Proof of Proposition~\ref{parsuff}.}
\label{teorema}
Let $\Pi_{k\setminus i}=\Pi_k\setminus \{i\}$. If $Y_i$ is parent of $Y_k$ we have that
\begin{align}
p'_{ky_k\pi'_k}&=P(y_k\;|\; \bm{y}_{\Pi'_k})=P(y_k\;|\; \bm{y}_{\Pi_i},\bm{y}_{\Pi_{k\setminus i}})
=\sum_{y_i\in\mathcal{Y}_i} P(y_k\;|\; \bm{y}_{\Pi_k},\bm{y}_{\Pi_i})P(y_i\;|\; \bm{y}_{\Pi_{k\setminus i}},\bm{y}_{\Pi_i})\label{nonso}\\
&=\sum_{y_i\in\mathcal{Y}_i}P(y_k\;|\; \bm{y}_{\Pi_k})P(y_i\;|\; \bm{y}_{\Pi_i})= \sum_{y_i\in\mathcal{Y}_i}p_{ky_k\pi_k}p_{iy_i\pi_i}\nonumber
\end{align}
If $Y_i$ is a parent but not the parent of $Y_j$, then $P(y_i\;|\; \bm{y}_{\Pi_{j\setminus i}},\bm{y}_{\Pi_i})$ as in equation~(\ref{nonso}) can be written as
\begin{align*}
P(y_i\;|\; \bm{y}_{\Pi_{j\setminus i}},\bm{y}_{\Pi_i})&=P(y_i\;|\;\bm{y}_{\Pi_{j}\setminus [i-1]},\bm{y}_{\Pi_{j}\cap [i-1]},\bm{y}_{\Pi_i})\\
&=\frac{P(\bm{y}_{\Pi_{j}\setminus [i-1]}\;|\;y_i,\bm{y}_{\Pi_{j}\cap[i-1]},\bm{y}_{\Pi_i})P(y_i\;|\;\bm{y}_{\Pi_i})}{\sum_{y_i\in\mathcal{Y}_i}P(\bm{y}_{\Pi_{j}\setminus [i-1]}\;|\;y_i,\bm{y}_{\Pi_{j}\cap [i-1]},\bm{y}_{\Pi_i})P(y_i\;|\;\bm{y}_{\Pi_i})}\\
&=\frac{\prod_{l\in \Pi_{j}\setminus[i-1]}\sum_{\bm{\mathcal{Y}}_{\Pi_i\cap\Pi_j\cap\Pi_l}} P(y_l\;|\; \bm{y}_{\Pi_l})P(y_i\;|\;\bm{y}_{\Pi_i})}{\sum_{y_i\in\mathcal{Y}_i}\prod_{l\in \Pi_j\setminus[i-1]}\sum_{\bm{\mathcal{Y}}_{\Pi_i\cap\Pi_j\cap\Pi_l}} P(y_l\;|\; \bm{y}_{\Pi_l})P(y_i\;|\;\bm{y}_{\Pi_i})},
\end{align*}
\subsection{Proof of Theorem~\ref{polyasy}}
\label{cippu}
For $i,j,k,l\in\mathbb{V}$ and $s,t\in [m]$, an asymmetry $Y_i=y_i\Rightarrow Y_j=y_j$ implies that any monomials that include terms of the form $p_{ky_k\pi_k}$, $\psi_{s\pi_s}$, $p_{ky_k\pi_k}p_{ly_l\pi_l}$, $\psi_{t\pi_t}\psi_{s\pi_s}$ and $p_{ky_k\pi_k}\psi_{s\pi_s}$ entailing both instantiations $y_i$ and $y_j$ are associated to a non possible combination of events, with $y_k\in\mathcal{Y}_k$, $\pi_k\in\bm{\mathcal{Y}}_{\Pi_k}$, $y_l\in\mathcal{Y}_l$, $\pi_l\in\bm{\mathcal{Y}}_{\Pi_l}$, $\pi_t\in\bm{\mathcal{Y}}_{P_t}$ and  $\pi_s\in\bm{\mathcal{Y}}_{P_s}$. Thus these monomials have to be set equal to zero. 

For $j<t\leq z$, $\bar{\bm{U}}_t$ has an associated set $B_t$ which includes both $i$ and $j$ and consequently $\prod_{s\in B_t\setminus \{i\cup j\}}r_s$ rows of the vector corresponds to the conditioning on $Y_i=y_i$ and $Y_j=y_j$. Therefore all the monomials in those rows have to be set equal to zero.

For $i<t\leq j$, the index $i$ is in the set $B_t$, whilst the variable $Y_j$ has been already EUMarginalized. Thus, there are only $\prod_{s\in B_t\setminus \{i\}}r_s$ rows conditional on the event $Y_i=y_i$. In those rows only some of the monomials are associated to the event $Y_j=y_j$. Specifically, the ones implying $Y_j=y_j$ can only be multiplying a term including a $\bm{\psi}_{xP_x}$ from a utility vertex $U_x$ subsequent to $Y_j$ in the MID DS. We can deduce that there are 
$
\prod_{s=t}^{j_a}r_s/r_j
$ monomials of degree $d_{tba}$ that include the case $Y_j=y_j$ in such entries of $\bar{\bm{U}}_t$, for $a=x,\dots,m$ and $b=l,\dots,a$ (using the notation of Theorem~\ref{polyexp}).

Lastly, if $t\leq i$, then the set $B_t$ does not include $i$ and $j$, which have been both EUMarginalized. Thus monomials including a combination of the events $Y_j=y_j$ and $Y_i=y_i$ appears in each row of $\bar{\bm{U}}_t$. Similarly as before, we can deduce that there are $\prod_{s=t}^{j_a}r_s/(r_i\cdot r_j)$ monomials of degree $d_{tba}$, $a=x,\dots,m$, $b=l,\dots,a$, implying the event $Y_i=y_i\land Y_j=y_j$.

\section{Maple Code}
\label{maple}

\subsection{Initialization functions}
\label{b1}
\begin{alltt}
### Required Packages ###
with(ArrayTools): with(LinearAlgebra):
\end{alltt}

\begin{alltt}
### Computation of the highest index in each parent set of a utility node ###
# \textbf{Inputs}: PiU::table, parent sets of utility nodes; m::integer, num. utility nodes
# \textbf{Output}: J::list
\end{alltt}

\begin{alltt}
CompJ := \textbf{proc}(PiU,m) \textbf{local} i,j:
\textbf{for} j  \textbf{to} m \textbf{do} J[j] := max(PiU[J]) \textbf{end do}:
\textbf{return} convert(J,list):
\textbf{end proc}:
\end{alltt}

\begin{alltt}
### Computation of the indices of the argument of the EU at step i ###
# \textbf{Inputs}: PiU::table; PiV::table, parent sets of random nodes; i::integer;
# n::integer, number of random nodes; J::list
# \textbf{Output}: Bi[i]::set
\end{alltt}

\begin{alltt}
CompBi := \textbf{proc}(PiU,PiV,i,n,J) \textbf{local} Bi,part,j:
Bi[i], part := \{\},\{\}:
\textbf{for} j \textbf{from} i \textbf{to} n \textbf{do}
 part := part \textbf{union} \{j\}:
 \textbf{if} member(j,V) \textbf{then} Bi[i] := Bi[i] \textbf{union} PiV[j] \textbf{end if}:
 \textbf{if} member(j,J,'l') \textbf{then} Bi[i] := Bi[i] \textbf{union} PiU[l] \textbf{end if}:
\textbf{end do}:
Bi[i] := Bi[i] \textbf{minus} part:
\textbf{return} Bi[i]:
\textbf{end proc}:
\end{alltt}

\begin{alltt}
### Initialization of an MID ###
# \textbf{Inputs}: p::table, probability vectors; psi::table, utility vectors; PiV::table;
# PiU::table; n::integer; m::integer
# \textbf{Outputs}: J::list; Bi::list; u::table, EU vectors
\end{alltt}

\begin{alltt}
Initialize := \textbf{proc}(p, psi, PiV, PiU, n, m) \textbf{local} J, i, Bi, u:
J := CompJ(PiU, m):
\textbf{for} i \textbf{to} n \textbf{do}  Bi[i] := CompBi(PiU, PiV, i, n, J) \textbf{end do}:
Bi[n+1], u[n+1] := \{\}, []:
\textbf{return} J, Bi, u:
\textbf{end proc}:
\end{alltt}

\begin{alltt}
### Identification of an optimal policy (at random) ###
# \textbf{Inputs}: r::table; i::integer, index of the decision variable,
# t::integer, number of random draws
#\textbf{Outputs}: maxi::vector, optimal decisions
\end{alltt}

\begin{alltt}
Maximize := \textbf{proc}(r, i, t) \textbf{local} maxi, l: 
maxi := Vector(t, 0):
\textbf{for} l \textbf{to} t \textbf{do} maxi[l] := RandomTools[Generate](integer(range = 1 .. r[i])) \textbf{end do}:
\textbf{return} maxi:
\textbf{end proc}:
\end{alltt}

\subsection{EU duplications}
\label{b2}
\begin{alltt}
### EUDuplication of a utility vector and an EU vector ###
# \textbf{Inputs}: u::table; psi::table; j::integer; PiV::table; PiU::table;
# r::table, size of the decision and sample spaces; Bi::table; J::list
# \textbf{Outputs}: utemp::list, EUDuplicated version of u;
# psitemp::list, EUDuplicated version of psi
\end{alltt}

\begin{alltt}
EUDuplicationPsi := \textbf{proc}(u, psi, j, PiV, PiU, r, Bi, J)
\textbf{local} i, uprime, psip, psit, utemp, x, sx, y, l, z: 
i := max(PiU[j]):
uprime, psip, psit, utemp := [], [], psi[j], u[i+1]: 
\textbf{for} x \textbf{from} max(Bi[i+1], PiU[j]) \textbf{by} -1 \textbf{to} 1 \textbf{do}
 \textbf{if} member(x, (PiU[j] \textbf{union} Bi[i+1]) \textbf{minus} (PiU[j] \textbf{intersect} Bi[i+1])) \textbf{then}
  sx := 1:
  \textbf{for} y \textbf{from} x+1 \textbf{to} max(Bi[i+1], PiU[j]) \textbf{do}
   \textbf{if} member(y, \textbf{union}(Bi[i+1], PiU[j])) \textbf{then} sx := sx*r[y] \textbf{end if}
  \textbf{end do}:
  \textbf{if} member(x, Bi[i+1]) \textbf{then for} l \textbf{to} Size(psit)[2]/sx \textbf{do} \textbf{for} z \textbf{to} r[x] \textbf{do}
   psip := [op(psip),op(convert(convert(psit,list)[(l-1)*sx+1..l*sx],list))]
  \textbf{end do  end do}:
  psit, psip := psip, []:
  \textbf{elif} member(x, PiU[j]) \textbf{then for} l \textbf{to} Size(utemp)[2]/sx \textbf{do for} z \textbf{to} r[x] \textbf{do} 
   uprime:=[op(uprime),op(convert(convert(utemp,list)[(l-1)*sx+1..l*sx],list))]
  \textbf{end do end do:} 
  utemp, uprime := uprime, []:
\textbf{end if end if end do}:
\textbf{return} utemp, psit:
\textbf{end proc}:
\end{alltt}

\begin{alltt}
### EUDuplication of a probability vector and an EU vector ###
# \textbf{Inputs}: u::table; p::table; i::integer; PiV::table; PiU::table; 
  r::table; Bi::table; J::list
# \textbf{Outputs}: utemp::list, EUDuplicated version of u;
# ptemp::list, EUDuplicated version of p
\end{alltt}

\begin{alltt}
EUDuplicationP := \textbf{proc} (u, p, i, PiV, PiU, r, Bi, J)
\textbf{local} uprime, pprime, ptemp, utemp, x, sx, y, l, z, Uni:
uprime, pprime, ptemp, utemp := [], [], p[i], u[i+1]:
\textbf{if} member(i, J) \textbf{then} member(i, J, 'j');
 Uni := (Bi[i+1] union PiV[i]) union PiU[j]:
 \textbf{for} x \textbf{from} max(Uni) \textbf{by} -1 \textbf{to} 1 \textbf{do}
  \textbf{if} member(x, Uni \textbf{minus} ((Bi[i+1] \textbf{union} PiU[j]) \textbf{intersect} (PiV[i] \textbf{union} {i}))) \textbf{then}
   sx := 1;
   \textbf{for} y \textbf{from} x+1 \textbf{to} max(Uni) \textbf{do if} member(y, Uni) \textbf{then}
    sx := sx*r[y] \textbf{end if end do}; 
   \textbf{if} member(x, \textbf{union}(Bi[i+1], PiU[j])) \textbf{then}
    \textbf{for} l \textbf{to} Size(ptemp)[2]/sx \textbf{do for} z \textbf{to} r[x] \textbf{do} 
     pprime:=[op(pprime),op(convert(convert(ptemp,Array)[(l-1)*sx+1..l*sx],list))]
    \textbf{end do end do}:
    ptemp, pprime := pprime, []:
   \textbf{elif} member(x, PiV[i]) \textbf{then for} l \textbf{to} Size(utemp)[2]/sx \textbf{do} 
    \textbf{for} z \textbf{to} r[x] \textbf{do} 
     uprime:=[op(uprime),op(convert(convert(utemp,Array)[(l-1)*sx+1..l*sx],list))] 
    \textbf{end do end do}:
    utemp, uprime := uprime, []:
\textbf{end if end if end do}:
\textbf{else for} x \textbf{from} max(Bi[i+1], PiV[i]) \textbf{by} -1 \textbf{to} 1 \textbf{do}
 \textbf{if} member(x,(Bi[i+1] \textbf{union} PiV[i])\textbf{minus}(Bi[i+1] \textbf{intersect} (PiV[i] \textbf{union} {i}))) \textbf{then}
  sx := 1;
  \textbf{for} y \textbf{from} x+1 \textbf{to} max(Bi[i+1],PiV[i]) \textbf{do if} member(y,Bi[i+1] \textbf{union} PiV[i]) \textbf{then}
   sx := sx*r[y]
  \textbf{end if end do}:
  \textbf{if} member(x, Bi[i+1]) \textbf{then for} l \textbf{to} Size(ptemp)[2]/sx \textbf{do for} z \textbf{to} r[x] \textbf{do} 
   pprime:=[op(pprime),op(convert(convert(ptemp, Array)[(l-1)*sx+1..l*sx],list))] 
  \textbf{end do} \textbf{end do}:
 ptemp, pprime := pprime, []:
\textbf{elif} member(x, PiV[i]) \textbf{then for} l \textbf{to} Size(utemp)[2]/sx \textbf{do for} z \textbf{to} r[x] \textbf{do}
 uprime:=[op(uprime),op(convert(convert(utemp,Array)[(l-1)*sx+1..l*sx],list))]
\textbf{end do end do};
utemp, uprime := uprime, []:
\textbf{end if end if end do end if}:
utemp, ptemp := convert(utemp,Array), convert(ptemp,Array):
\textbf{return} utemp,ptemp:
\textbf{end proc}:
\end{alltt}

\subsection{EU operations}
\label{b3}
\begin{alltt}
### EuMultiSum between an EU vector and a utility vector ###
# \textbf{Inputs}: u::table; psi::table; j::integer; PiV::table; PiU::table;
# r::table; Bi::table; J::list
# \textbf{Outputs}: ut::list, EU vector after an EUMultiSum  
\end{alltt}

\begin{alltt}
EUMultiSum := \textbf{proc}(u, psi, j, PiV, PiU, r, Bi, J) \textbf{local} i, uprime, psip, ut;
i := max(PiU[j]);
\textbf{if} j = Size(convert(PiU, list), 2) \textbf{then} ut := k[j]*~psi[j]:
\textbf{else} uprime, psip := EUDuplicationPsi(u, psi, j, PiV, PiU, r, Bi, J);
ut := h*~k[j]*~psip*~uprime +~ uprime +~ k[j]*~psip \textbf{end if}:
\textbf{return} ut:
\textbf{end proc}:
\end{alltt}

\begin{alltt}
### EUMarginalization over a sample space ###
# \textbf{Inputs}: u::table; p::table; i::integer; PiV::table; PiU::table; 
  r::table; Bi::table; J::list
# \textbf{Outputs}: ut::list, EU vector after EUMarginalization
\end{alltt}

\begin{alltt}
EUMarginalization := \textbf{proc} (u, p, i, PiV, PiU, r, Bi, J)
\textbf{local} uprime, pprime, ut, cols, l, k:
uprime, pprime := EUDuplicationP(u, p, i, PiV, PiU, r, Bi, J):
cols :=  Size(pprime)[2]:
ut := convert(ZeroVector(cols/r[i]), Array):
\textbf{for} l \textbf{to} (cols/r[i]) \textbf{do for} k \textbf{to} r[i] \textbf{do} 
 ut[l] := ut[l]+ pprime[r[i]*(l-1)+k]*uprime[r[i]*(l-1)+k]:
\textbf{end do end do}:
\textbf{return} ut:
\textbf{end proc}:
\end{alltt}

\begin{alltt}
### EUMaximization over a decision space ###
# \textbf{Inputs}: u::table; i::integer; r::table
# \textbf{Outputs}: u[i]::list, EU vector after EUMaximization
\end{alltt}

\begin{alltt}
EUMaximization := \textbf{proc}(u, i, r) \textbf{local} opt,l ;
opt := Maximize(r, i, Size(u[i+1])[2]/r[i]);
u[i] := Array([seq(0,l in 1..Size(opt)[1])]):
\textbf{for} l \textbf{to} Size(opt)[1] \textbf{do}
 u[i][l] := convert(u[i+1],Array)[r[i]*(l-1)+opt[l]] \textbf{end do};
\textbf{return} u[i]:
\textbf{end proc}:
\end{alltt}

\subsection{The symbolic algorithm}
\label{b4}
\begin{alltt}
### Symbolic evaluation algorithm for an MID ###
# \textbf{Inputs}: p::table; psi::table; PiV::table; PiU::table; n::integer; m::integer;
# De::set, index set of the decision variables;
# V::set, index set of the random variables; r::table
# \textbf{Output}: eu::table, EU vectors;
\end{alltt}

\begin{alltt}
SymbolicExpectedUtility := \textbf{proc}(p, psi, PiV, PiU, n, m, De, V, r)
\textbf{local} J, Bi, utemp, i, j, eu;
J, Bi, eu := Initialize(p, psi, PiV, PiU, n, m);
j := m;
\textbf{for} i \textbf{from} n \textbf{by} -1 \textbf{to} 1 \textbf{do if} j = 0 \textbf{then if} member(i, De) \textbf{then}
 eu[i] := EUMaximization(eu, i, r)
\textbf{else} eu[i] := EUMarginalization(eu, p, i, PiV, PiU, r, Bi, J) \textbf{end if};
\textbf{else} \textbf{if} J[j] = i then if member(i, De) \textbf{then}
 utemp[i+1] := EUMultiSum(eu, psi, j, PiV, PiU, r, Bi, J);
 eu[i] := EUMaximization(utemp, i, r)
\textbf{else}
 utemp[i+1] := EUMultiSum(eu, psi, j, PiV, PiU, r, Bi, J);
 eu[i] := EUMarginalization(utemp, p, i, PiV, PiU, r, Bi, J)
\textbf{end if};
j := j-1
\textbf{else if} member(i, De) \textbf{then} eu[i] := EUMaximization(eu, i, r)
\textbf{else} eu[i] := EUMarginalization(eu, p, i, PiV, PiU, r, Bi, J) \textbf{end if}
\textbf{end if end if end do};
\textbf{return} eu:
\textbf{end proc}:
\end{alltt}

\subsection{Implementation of the example}
\label{b5}
\noindent Consider the MID in Fig.~\ref{fig-ex} with $n=6$ variables (decision or random nodes) and $m=3$ utility nodes.
\begin{alltt}
 ### Definition of the MID ###
# number of variables and  utility nodes
    n := 6: m := 3: 
# V contains the indices of random nodes and De  those of the decision nodes
    V := {2, 3, 5, 6}: De := {1, 4}:
# Conditional probabilities
    p[6] := [p6111, p6011, p6101, p6001, p6110, p6010, p6100, p6000]:
    p[5] := [p5111, p5011, p5101, p5001, p5110, p5010, p5100, p5000]:
    p[3] := [p3111, p3011, p3101, p3001, p3110, p3010, p3100, p3000]:
    p[2] := [p211, p201, p210, p200]:
# Utility parameters
    psi[1] := [psi11, psi10]:
    psi[2] := [psi21, psi20]:
    psi[3] := [psi311, psi301, psi310, psi300]:
# Parents of random nodes
    PiV[2] := {1}:  PiV[3] := {1, 2}:  PiV[5] := {3, 4}: PiV[6] := {4, 5}:
# Parents of utility nodes
    PiU[1] := {3}: PiU[2] := {5}: PiU[3] := {4, 6}:
# Number of levels of the variables
    r[1] := 2: r[2] := 2: r[3] := 2: r[4] := 2: r[5] := 2: r[6] := 2: 
### Computation of the EU vectors ###
    eu := SymbolicExpectedUtility(p, psi, PiV, PiU, n, m, De, V, r):
\end{alltt}

\noindent Example of the output of \verb|eu[1]|:
\begin{alltt}
[((k[1]*psi11+h*k[1]*psi11*((k[2]*psi21+h*k[2]*psi21*
(p6010*psi300*k[3]+p6110*psi310*k[3])+k[3]*psi300*p6010
+k[3]*psi310*p6110)*p5101+(k[2]*psi20+h*k[2]*psi20*(p6000*psi300*k[3]
+p6100*psi310*k[3])+k[3]*psi300*p6000+k[3]*psi310*p6100)*p5001)+
(k[2]*psi21+h*k[2]*psi21*(p6010*psi300*k[3]+p6110*psi310*k[3])
+k[3]*psi300*p6010+k[3]*psi310*p6110)*p5101+
(k[2]*psi20+h*k[2]*psi20*(p6000*psi300*k[3]+p6100*psi310*k[3])
+k[3]*psi300*p6000+k[3]*psi310*p6100)*p5001)*p3110+
(k[1]*psi10+h*k[1]*psi10*((k[2]*psi21+h*k[2]*psi21*(p6011*psi301*k[3]
+p6111*psi311*k[3])+k[3]*psi301*p6011+k[3]*psi311*p6111)*p5110+
(k[2]*psi20+h*k[2]*psi20*(p6001*psi301*k[3]+p6101*psi311*k[3])
+k[3]*psi301*p6001+k[3]*psi311*p6101)*p5010)+
(k[2]*psi21+h*k[2]*psi21*(p6011*psi301*k[3]+p6111*psi311*k[3])
+k[3]*psi301*p6011+k[3]*psi311*p6111)*p5110+
(k[2]*psi20+h*k[2]*psi20*(p6001*psi301*k[3]+p6101*psi311*k[3])
+k[3]*psi301*p6001+k[3]*psi311*p6101)*p5010)*p3010)*p210+
((k[1]*psi11+h*k[1]*psi11*((k[2]*psi21+h*k[2]*psi21*(p6010*psi300*k[3]
+p6110*psi310*k[3])+k[3]*psi300*p6010+k[3]*psi310*p6110)*p5101+
(k[2]*psi20+h*k[2]*psi20*(p6000*psi300*k[3]+p6100*psi310*k[3])
+k[3]*psi300*p6000+k[3]*psi310*p6100)*p5001)+
(k[2]*psi21+h*k[2]*psi21*(p6010*psi300*k[3]+p6110*psi310*k[3])
+k[3]*psi300*p6010+k[3]*psi310*p6110)*p5101+
(k[2]*psi20+h*k[2]*psi20*(p6000*psi300*k[3]+p6100*psi310*k[3])
+k[3]*psi300*p6000+k[3]*psi310*p6100)*p5001)*p3100+
(k[1]*psi10+h*k[1]*psi10*((k[2]*psi21+h*k[2]*psi21*
(p6011*psi301*k[3]+p6111*psi311*k[3])+k[3]*psi301*p6011+k[3]*psi311*p6111)*p5110+
(k[2]*psi20+h*k[2]*psi20*(p6001*psi301*k[3]+p6101*psi311*k[3])
+k[3]*psi301*p6001+k[3]*psi311*p6101)*p5010)+
(k[2]*psi21+h*k[2]*psi21*(p6011*psi301*k[3]+p6111*psi311*k[3])
+k[3]*psi301*p6011+k[3]*psi311*p6111)*p5110+
(k[2]*psi20+h*k[2]*psi20*(p6001*psi301*k[3]+p6101*psi311*k[3])
+k[3]*psi301*p6001+k[3]*psi311*p6101)*p5010)*p3000)*p200]
\end{alltt}
%
%

\end{document}